\documentclass{svproc}
\usepackage{url}

\usepackage{floatrow}
\usepackage[rflt]{floatflt}
\usepackage{float}
\usepackage{amsmath}
\usepackage{amsbsy}
\usepackage{amssymb}
\usepackage{textcomp}
\usepackage{graphicx}
\usepackage{amsfonts}
\usepackage{multirow}
\usepackage{url}
\usepackage{hyperref}
\usepackage{longtable}
\usepackage{lscape}
\usepackage{hhline}
\usepackage[utf8]{inputenc}
\usepackage{colortbl,array} 
\usepackage{multirow,bigdelim}
\usepackage{booktabs}
\usepackage[linesnumbered,ruled]{algorithm2e}
\usepackage{xcolor,colortbl}
\usepackage[mathscr]{eucal}
\usepackage{mathtools, nccmath}
\usepackage{tikz}
\usepackage{pgfplots}
\usepackage{pgfplotstable}
\usepgfplotslibrary{groupplots}
\usetikzlibrary{pgfplots.groupplots}
\pgfplotsset{compat=1.3}
\usepackage{subcaption}
\usepackage{wrapfig}
\definecolor{Gray}{gray}{0.8}

\def\ksvd{{$K$-SVD }}

\def\kmeans{{$K$-means }}
\newcommand{\argmin}{arg\,min}
\SetKwComment{Comment}{$\triangleright$~~}{}

\begin{document}
\mainmatter              
\title{Sparse Dictionary Learning for Image Recovery by Iterative Shrinkage}
\titlerunning{Sparse Dictionary Learning for Image Recovery}  
%

\author{Shima Shabani \and Mohammadsadegh Khoshghiaferezaee \and Michael Breuß}
%
\authorrunning{Shima Shabani et al.} 
%

\tocauthor{Shima Shabani, Mohammadsadegh Khoshghiaferezaee, and Michael Breuß}

\institute{Institute for Mathematics, Brandenburg Technical University\\
Platz der Deutschen Einheit 1, 03046 Cottbus, Germany,\\
\email{\{shima.shabani,khoshmoh,breuss\}@b-tu.de}
}

\maketitle              

\begin{abstract}
In this paper we study the sparse coding problem in the context of sparse dictionary learning for image recovery.
To this end, we consider and compare several state-of-the-art sparse optimization methods constructed using the shrinkage operation. 
As the mathematical setting of these methods, we consider an online approach as algorithmical basis together with the basis pursuit denoising problem that arises by the convex optimization approach to the dictionary learning problem.
 
By a dedicated construction of datasets and corresponding dictionaries, we study the effect of enlarging the underlying learning database on reconstruction quality making use of several error measures. 
Our study illuminates that the choice of the optimization method may be practically important in the context of availability of training data. 
In the context of different settings for training data as may be considered part of our study, we illuminate the computational efficiency of the assessed optimization methods.
\end{abstract}
\section{Introduction}
Inspired by the sparsity mechanism of the human vision system \cite{OF}, \emph{sparse dictionary learning} (SDL) is a representation learning method that aims to model input data, such as an image, in the form of a linear combination of a few elements, called atoms, that is close to the input. Atoms compose the dictionary, and such a sparse combination occurs in an overcomplete dictionary.
The underlying mechanism of overcompleteness is that a dictionary with significantly more elements than the original dimensionality may enable data representation to be more sparse than if we had just enough elements to span the data space.
Research has demonstrated that learning a desired dictionary by forming atoms from training data leads to state-of-the-art results in many practical applications and yields better performance than by choosing a preconstructed one (e.g., using wavelets). Here, we refer to some applications in denoising \cite{FLSS,GL}, clustering \cite{CSPPC,JNZ}, image deblurring \cite{MMYZ,XMWPZ}, image super-resolution \cite{YWLCH,YWHM}, face recognition \cite{HDRL,JWZDML}, image segmentation \cite{CLJ}, and visual computing \cite{ZL}.

Given a training dataset, the core idea of identifying an overcomplete dictionary best fitting to the data, was studied initially by Field and Olshausen in 1996 \cite{OF,OF1}. They were motivated by an analogy between the atoms of a dictionary and the population of simple cells in the visual cortex. Their research has shown that a neural network equipped with a learning algorithm for sparse reconstruction of natural images develops atoms with properties similar to the receptive fields of simple cells in the primary visual cortex. In the visual cortex, the neurons' receptive fields and activities are comparable to the atoms from the dictionary and their corresponding coefficients for input data, respectively. 

For a given training dataset, SDL methods try to find an adapted dictionary and sparse representations simultaneously. At the technical level, this means to find the approximate decomposition of the training matrix, in terms of its columns as separate training objects, giving the elements of the dictionary as well as the corresponding sparse representation matrix of the training set. This amounts to a minimization problem with two stages: \emph{dictionary updating} and \emph{sparse coding}. Various methods deal with these two stages in different ways \cite{Ela}.  

\subsubsection{Related Work}
Tackling the sparse dictionary learning problem, the method of optimal directions (MOD) \cite{EAH} is one of the first methods. Here, the key idea is solving the minimization problem subject to the limited number of nonzero components for the sparse representation of the data set.
As a general form of the independent component analysis technique related to singular value decomposition (SVD), \cite{LS} proposes a process for learning an overcomplete basis by considering it as a probabilistic model of the observed data. 
The method in \cite{LGBB} learns the sparse overcomplete dictionaries as unions of orthonormal bases. It estimates sparse coefficients based on a variant of block coordinate relaxation algorithm \cite{SBT}, followed by updating the chosen dictionary atom via SVD. The \ksvd method \cite{AEB}, generalization of the \kmeans \cite{GG}, performs SVD at its core to update the atoms of the dictionary one by one like that in \cite{LGBB} and enforces input data encoding by the contribution of a limited number of atoms in a way identical to the MOD approach. 

\ksvd is a batch procedure, accessing the whole training set for the minimization process at each iteration. So, handling a large-sized or dynamic training set like in video sequences requires more effort. To address this issue, \cite{MBPS} presents an online approach that processes one data or a small subset of the training set at a time. That is helpful in the usual content of dictionary learning adapted to small patches of the training data with several millions of these patches. It also presents a dictionary updating stage with low memory consumption and lower computational cost without explicit learning rate tuning compared to the classical batch algorithms. 

Turning to the algorithmic solution of the minimization problems arising in the sparse coding stage, an important design choice in sparse dictionary learning is given by the selection of the optimization method that enables to enforce the sparsity. 
As previous investigations performed in the context of sparse recovery have demonstrated, see e.g. \cite{ESK,MMNCY,SB}, the choice of the optimization method may have profound influence on accuracy of results and required computation time. Both of which may greatly influence the overall efficiency of sparse dictionary learning, depending on dataset size and field of application.

\subsubsection{Our Contribution}
Motivated by the beneficial properties of the online dictionary learning method presented in \cite{MBPS}, we consider algorithmic approaches for sparse coding in that framework. Thereby we focus on image recovery as this is a fundamental problem in the field. 
More concretely, we present a detailed computational study comparing several state-of-the-art optimization methods for that purpose. 
In doing this, we also illuminate the impact of the underlying dataset on the construction of the dictionary.
The methods we consider are technically based on the shrinkage operator, as detailed below, and include classic efficient methods like e.g. {\tt FISTA} \cite{BT} as well as more recent schemes.
The sparse coding problem tackled here takes on the form of \emph{basis pursuit denoising}, compare e.g. \cite{Ela}.

In total, our study highlights important computational aspects like e.g. parameter choices in implementations as well as accuracy properties and efficacy of computed solutions.
In particular, we study how error measures for image recovery develop when increasing the database for SDL. 
We conjecture that our dedicated proceeding and dataset as well as dictionary construction sheds light on properties of the methods and enables a practical choice of methods.

\section{Sparse Dictionary Learning}
Sparse dictionary learning produces an overcomplete dictionary $D\in\mathbb{R}^{m\times n}$ ($m\gg n$) with atoms $\{d_i\}_{i=1}^{n}$ and $d_i\in \mathbb{R}^{m}$, 
to span the data set. That improves the sparsity and flexibility of the input data representation.
The sparse representation of an input like $p\in \mathbb{R}^{m}$ can be exact or approximate regarding $D$. Exact sparse representation $x\in \mathbb{R}^{n}$ is the solution of
\begin{equation}
    \label{eq:spa_def_0}
  \min_{x}~\lVert x \rVert_0  \quad \text{w.r.t.} \quad  p=Dx
\end{equation}
while
\begin{equation}
    \label{eq:spa_def_1}
  \min_{x}~\lVert x \rVert_0  \quad \text{w.r.t.} \quad   \lVert p - Dx \rVert_2\leq \epsilon 
\end{equation}
gives the approximate representation.
Here $\lVert \cdot\rVert_0$ measures the number of the nonzero entries for a vector, $\lVert \cdot \rVert_2$ is the 2-norm, and the known parameter $\epsilon$ is a toleration for the representation error.
Some classical dictionary learning techniques \cite{AEB,OF1} optimize the following empirical problem based on the training database 
$\{p_i\}_{i=1}^{N}$, $p_i\in \mathbb{R}^{m}$ 
\begin{equation}
    \label{eq:base_problem}
  \min_{D,\,x_i} \sum_{i = 1}^{N} \lVert p_i - Dx_i \rVert_2^2  \quad \text{w.r.t.} \quad  \lVert x_i \rVert_0\leq T
\end{equation}
where $T$ controls the number of nonzero elements in the sparse coding stage. In the learning process, (\ref{eq:base_problem}) aims to find the dictionary and proper sparse representations of the training set jointly. 

One should note that $\lVert \cdot\rVert_0$ is an easily grasped notion of sparsity but not necessarily the right notion for empirical work. 
Real world data would rarely be representable by a vector of coefficients containing many zeroes. One popular way to render sparsity more tractable is to relax the highly discontinuous function $\lVert \cdot\rVert_0$, replacing it with a continuous approximation. In the sparse coding stage, \cite{LBRN,MBPS} consider $\lVert \cdot\rVert_1$ instead of $\lVert \cdot\rVert_0$ for the joint minimization problem  
\begin{equation}
    \label{eq:base_problem1}
  \min_{D,\,x_i} \frac{1}{N}\sum_{i = 1}^{N}\Big(\frac{1}{2}\lVert p_i - Dx_i \rVert_2^2+\mu\lVert x_i \rVert_1\Big)
\end{equation}
with regularize parameter $\mu$. The $\ell_1$ penalty still yields a sparse solution, \cite{Ela} gives a geometrical interpretation for that. 

In general, the $\ell_0$-minimization problem like (\ref{eq:base_problem}) is known to be NP-hard. Relaxing $\ell_0$ to $\ell_1$ leads to a computationally much more convenient convex optimization problem. The problem (\ref{eq:base_problem1}) is convex with respect to each variable when the other one is fixed. So, an approach to solving this problem is to alternate between the two variables, minimizing over one while keeping the other fixed. 

It is common to restrict the Euclidean norm of the atoms to prevent $D$ from being arbitrarily large. Otherwise, it leads to small values for $x_i$. Therefore, in (\ref{eq:base_problem1}), we suppose $D\in \mathcal{C}$ which is a convex set
\begin{equation}\label{e:Cdef}
\mathcal{C} = \Big\{D=[d_1, \ldots, d_n]\in\mathbb{R}^{m\times n}|~ d_j\in\mathbb{R}^m,~  
 d_j^Td_j\leq 1, ~j=1, \ldots, n\Big\}.
\end{equation}
Sparse recovery stage with a fixed $D$ in \eqref{eq:base_problem1} is a convex $\ell_2$-$\ell_1$ problem known as \textit{basis pursuit denoising}. One advantage of dealing with this non-smooth convex problem is that there are some practical methods for solving it and speeding up learning algorithms.  

Online dictionary learning approach in \cite{MBPS} uses {\tt LARS} (least angle regression) method for sparse recovery and implements block coordinate descent with warm starts in dictionary updating. The dictionary updating stage is parameter-free, without any learning rate tuning. The authors showed convergence of the proposed method to a stationary point of the cost function with probability one.  
Subsequently, we review a standard sparse recovery method and its acceleration forms implemented in this paper.   


\subsection{Review of a Standard Sparse Recovery Method}

In the sparse recovery stage for \eqref{eq:base_problem1}, the aim is to solve
\begin{equation}\label{e:defprob1}
 \min_{x}~ \underbrace{\frac{1}{2}\|Dx-p\|_2^2}_{f(x)}+\mu\|x\|_1 
\end{equation}
for a fixed $D$. Seeking the minimum of \eqref{e:defprob1}, one baseline is iteratively to generate a sequence $\{x_k\}_{k\geq 0}$ by
\begin{equation}\label{e:iterscheme}
x_{k+1} \, = \, x_k+\alpha_k d_k
\end{equation}
where $d_{k}$ is a descent direction and $\alpha_{k}$ a positive step size. 
As shown in \cite{SB},
\begin{equation}\label{e:dks}
d_k=\mathcal{S}_{\mu\tau_k}\Big(x_k-\tau_k \nabla f(x_k)\Big)-x_k,
\end{equation}
is a descent direction for \eqref{e:defprob1} generated by the nature of $\|.\|_1$. Here, $\tau_k>0$ is a suitable step size for the implicit gradient descent method, chosen in different ways within
some algorithms and
\begin{equation}\label{e:shri}
\mathcal{S}_{\lambda}(x)=\text{sgn}(x)\odot\max\Big\{|x|-\lambda, 0\Big\}
\end{equation}
is the \textit{shrinkage operator} \cite{PB} with factor
$\lambda>0$. In \eqref{e:shri},  $\text{sgn}(\cdot)$ stands for the signum
function and $\odot$ denotes the component-wise product, $(x\odot y)_i=x_iy_i$. This method has a global convergence rate $\mathcal{O}(\frac{1}{k})$ and is called \textit{iterative shrinkage-thresholding algorithm} ({\tt
ISTA})\ \cite{CW,DFM}. It is an attractive baseline method as it is still simple yet results in efficient algorithms.

The shrinkage operator has the following general property related to the minimization problem \eqref{e:defprob1}.
\begin{corollary}\label{cor1}
 $x^*$ is a minimum for \eqref{e:defprob1} if and only if
\begin{equation}\label{e:fpc}
x^*=\mathcal{S}_{\mu\tau_k}\Big(x^*-\tau_k \nabla f(x^*)\Big)
\end{equation}   
\end{corollary}
\begin{proof}
See \cite{PB}. 
\end{proof}
There are some proposed methods to accelerate {\tt ISTA}, which we implement for the sparse recovery stage of the dictionary learning method and image reconstruction by the produced dictionaries.

\subsection{Review of ISTA Acceleration Methods}\label{Acc}
This section briefly explains some  {\tt ISTA} acceleration algorithms implemented in this paper for the sparse recovery. 

\subsubsection{\tt FPC-BB:}
Based on Corollary \ref{cor1}, fixed-point continuation algorithm \cite{HYZ} uses an operator-splitting technique to solve a
sequence of problem \eqref{e:defprob1} and constructs the shrinkage step-size $\tau_k$ dynamically by the Barzilai-Borwein \cite{BB} technique.

\subsubsection{\tt TwIST:}
Two-step {\tt ISTA} \cite{BF} is a iteration method which relies on computing the next iteration based on two previously computed steps instead of one as
\begin{equation}
   x_{k+1}=(1-\alpha)x_{k-1}+(\alpha-\beta)x_k+\beta\mathcal{S}_{\mu\tau_k}\Big(x_k-\tau_k \nabla f(x_k)\Big) 
\end{equation}
 with $0<\alpha<2$ and $0<\beta<c\alpha$ for a fixed $c$.

\subsubsection{\tt FISTA:}
Fast iterative shrinkage-thresholding algorithm \cite{BT} is employed on a special linear combination of the previous two iterations with gradient Lipschitz constant $L$ as
\begin{equation}
\begin{split}
 x_{k}&=\mathcal{S}_{\frac{\mu}{L}}\Big(y_k-\frac{1}{L}\nabla f(x_k)\Big)\\
  t_{k+1}&=\frac{\sqrt{1+4t_k^2}}{2}\\
  x_{k+1}&=x_{k}+\big(\frac{t_k-1}{t_{k+1}}\big)(x_k-x_{k-1})
\end{split} 
\end{equation}
This method keeps computational simplicity of {\tt ISTA} 
and improves its global convergence rate to $\mathcal{O}(\frac{1}{k^2})$, compared to $\mathcal{O}(\frac{1}{k})$for {\tt ISTA}.

\subsubsection{\tt SpaRSA:}
Sparse reconstruction by separable approximation \cite{WNF} solves \eqref{e:defprob1} by exploring separable computational structures. In each iteration, it solves an optimization subproblem involving a quadratic term with diagonal Hessian (i.e., separable in the unknown) plus a sparse regularizer term.
Improved practical performance of {\tt SpaRSA} is the result of the variation of the shrinkage step size. It produces shrinkage step size dynamically by using the Barzilai-Borwein technique along with a non-monotone strategy. 

\subsubsection{\tt GSCG:} Generalized shrinkage conjugate gradient method \cite{ESM} presents a descent condition for the minimization problem \eqref{e:defprob1}. If such a condition holds, it uses an efficient descent direction generated by a combination of a generalized form of the conjugate gradient direction and the {\tt ISTA} descent direction \eqref{e:dks}. Otherwise, {\tt ISTA} 
is improved by a new form of $\tau_k$ that is a specific linear combination of the Barzilai-Borwein technique and the previous step size.

\subsubsection{\tt ISGA:} Iterative Shrinkage-Goldstein algorithm \cite{SB} uses \eqref{e:dks} with $\tau_k$ generated by the Barzilai-Borwein method to produce descent direction. Generating efficient step sizes $\alpha_k$ in \eqref{e:iterscheme}, {\tt ISGA} utilizes the Goldstein quotient under a sufficient descent condition, which speeds up the technique.

\section{Dictionary Learning Algorithm}

Considering online SDL as in \cite{MBPS} and the algorithms reviewed in Section \ref{Acc}, we produce different dictionaries corresponding to different sparse recovery methods for image reconstruction. Algorithm \ref{alg:one} gives an overview on resulting SDL methods. 

First, the algorithm chooses a solver to produce the corresponding dictionary (lines 1-2). At each iteration,  randomly, the algorithm draws one patch from the learning image set (line 5) and produces the desired dictionary based on the given number of patches in the initialization step ($N$). The main loop (lines 4-12) contains the sparse coding (lines 6-9) and dictionary updating stage (lines 10-11). Lines 8-9 of the main loop reset the past information related to the produced sparse coefficient for the patches. Thus, one needs some initialization (line 3) for that. In the sequel, we give more explanations for these steps. 

\subsubsection{Sparse Coding Stage}
 At iteration $k$ for the fixed dictionary $D_{k-1}$ of the previous iteration, line 7 computes sparse representation for the chosen patch $p_k$ using the considered sparse solver. 
Lines 8-9 save patch sparse information in
the matrices $A_k$ and $B_k$. Then, the algorithm updates the dictionary with updated $A_k$ and $B_k$ based on the $p_k$ sparse information.    

\subsubsection{Dictionary Updating Stage}
In this stage of iteration $k$, for the stored fixed sparse representation values of patches until step $k$, the algorithm updates the dictionary by using $D_{k-1}$ as a warm start. In the Appendix, we show the equality of both minimization problems in \eqref{e:algeq1}.

Like in \cite{MBPS}, the algorithm uses the block coordinate descent methods for updating via Algorithm \ref{alg:two}.
In the Appendix, we clarify how \eqref{e:algeq2} gives the solution of \eqref{e:algeq1} column-wise. Updating is only done for atoms with contributions in the sparse representation of the sparse coding stage, i.e., $A_{jj}\neq 0$ (a detail not mentioned in \cite{MBPS}). This process is column-wise, i.e., updating one atom while keeping the others fixed under the constraint $ d_j^Td_j\leq 1$. As shown in \cite{MBPS}, since this convex optimization problem admits separate constraints in the updated blocks (atoms), convergence to a global optimum is guaranteed. Similar to \cite{MBPS}, by warm restart with $D_{k-1}$, the algorithm does just one single iteration in this stage.

\begin{algorithm}
\caption{Dictionary Learning}\label{alg:one}
\KwData{$D_0\in\mathbb{R}^{m\times n}$ (initial dictionary), $N$ (number of the iterations),\\ $p\in\mathbb{R}^m$ (image patch), $\mu\in\mathbb{R}$ (regularization parameter)}
\KwResult{$D_N$ (learned dictionary)}
\vspace{2mm}
 $\text{List}=\{\tt{FISTA, FPC-BB, TwIST, GSCG, ISGA}\}$\ \Comment*[r]{Sparse recovery solvers}
 Select a {\tt{solver}} from the List\;
 $A_0 \gets 0,~~B_0 \gets 0$ (reset the past information)\;
 \vspace{2mm}
 \For{$k=1$ to $N$}{
  Draw $p_k$ from image set\;
  \vspace{2mm}
  \textbf{Sparse Coding Stage:}\\ 
  \vspace{1mm}
    $x_{k} = \underset{x}{\textrm{\argmin}}\frac{1}{2}\|D_{k-1}x-p_k\|_2^2+\mu\|x\|_1$\;  
  $A_k \gets A_{k-1}+x_kx_k^T$\;
  $B_k \gets B_{k-1}+p_kx_k^T$\;
  \vspace{2mm}
  \textbf{Dictionary Updating Stage:}\\
  \vspace{1mm}
  Update $D_k$ using \textbf{Algorithm 2} with $D_{k-1}$ as a warm restart, so that
  \begin{equation}\label{e:algeq1}
  \begin{split}
    D_k=&\underset{D\in \mathcal{C}}{\textrm{\argmin}}\frac{1}{k}\sum_{i=1}^{k} \frac{1}{2}\|Dx_i-p_i\|_2^2+\mu\|x_i\|_1\\
  =& \underset{D\in \mathcal{C}}{\textrm{\argmin}}\frac{1}{k}\big(\frac{1}{2}\text{Tr}(D^TDA_k)-\text{Tr}(D^TB_k)\big)   
  \end{split}  
  \end{equation}
 }
\end{algorithm}

\begin{algorithm}
\caption{Dictionary Updating}\label{alg:two}
\KwData{$D_{k-1}\in\mathbb{R}^{m\times n}$,\\ $A_k=[a_1,\ldots, a_n]=\sum_{i=1}^{k}x_kx_k^T\in\mathbb{R}^{n\times n}$, $B_k=[b_1,\ldots, b_n]=\sum_{i=1}^{k}p_kx_k^T\in\mathbb{R}^{m\times n}$.}
\vspace{2mm}
\KwResult{$D_k$}
\vspace{2mm}
 \While{convergence}{
 \For{$j=1$ to $n$}{
  Update the $j$-th column to optimize for (\ref{e:algeq1})
  \vspace{2mm}
  \begin{equation}\label{e:algeq2}
  \begin{split}
     u_j&\gets \frac{1}{A_{jj}}\big(b_j-Da_j\big)+d_j\\
  d_j&\gets \frac{1}{\max\big(\|u_j\|_2, 1\big)}u_j 
  \end{split}      
  \end{equation}
 }
 }
\end{algorithm}

\section{Experimental Validation}
In this section, we present experiments on natural images from the {\tt Kaggle} dataset \cite{Kag}. The experiments  were run in MATLAB R2024b on an egino BTO with 128 × Intel$^{\circledR}$ Xeon$^{\circledR}$ Gold 6430 and 125,1 GiB of RAM. 

\subsubsection{Data Preparation}
To learn dictionaries, we considered a colorful subset of images with the cardinality $12$ and resolution $244\times244$. We did some preprocessing to make them ready for the learning process. First, we transferred images to a grayscale and downsampled them to the resolution $112\times112$, Figure \ref{fig: Tr-Im}. To have a fair comparison for all sparse recovery solvers mentioned in the first line of Algorithm \ref{alg:one}, we selected the images of the learning set by given order in Figure \ref{fig: Tr-Im}, left to right row-wise. We also extracted ordered patches with $50\%$ overlapping and size $6\times6$ from the ordered images. $50\%$ overlapping means each patch has $50\%$ of its area in common with the adjacent patches. So, the process extracts $1296$ patches from each training image. Thus, patch-wise dictionary learning has been done over $15552$ ordered patches, i.e., $N=15552$ in Algorithm \ref{alg:one}. Based on the patch size, we considered the size of the produced overcomplete dictionaries as $36\times256$.  

\begin{figure}
    \begin{subfigure}[t]{.49\textwidth}
        \centering
        \includegraphics[width = 0.23\linewidth, height = 0.22\linewidth]{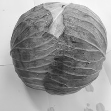}
        \includegraphics[width = 0.23\linewidth, height = 0.22\linewidth]{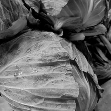} 
        \includegraphics[width = 0.23\linewidth, height = 0.22\linewidth]{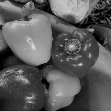} 
        \includegraphics[width = 0.23\linewidth, height = 0.22\linewidth]{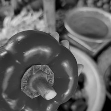}\\
        \vspace{0.3em}
        \includegraphics[width = 0.23\linewidth, height = 0.22\linewidth]{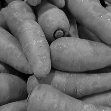} 
        \includegraphics[width = 0.23\linewidth, height = 0.22\linewidth]{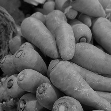} 
        \includegraphics[width = 0.23\linewidth, height = 0.22\linewidth]{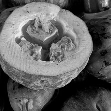}
        \includegraphics[width = 0.23\linewidth, height = 0.22\linewidth]{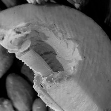}\\ 
        \vspace{0.3em}
        \includegraphics[width = 0.23\linewidth, height = 0.22\linewidth]{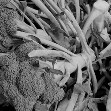} 
        \includegraphics[width = 0.23\linewidth, height = 0.22\linewidth]{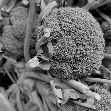} 
        \includegraphics[width = 0.23\linewidth, height = 0.22\linewidth]{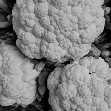}
        \includegraphics[width = 0.23\linewidth, height = 0.22\linewidth]{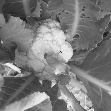}\\
    \caption{\label{fig: Tr-Im}
    Training images
    }
        
    \end{subfigure}
    \hfill
    \begin{subfigure}[t]{.49\textwidth}
        \centering
        \includegraphics[width = 0.23\linewidth, height = 0.22\linewidth]{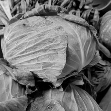}
        \includegraphics[width = 0.23\linewidth, height = 0.22\linewidth]{Fig1.1.png} 
        \includegraphics[width = 0.23\linewidth, height = 0.22\linewidth]{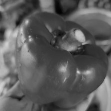} 
        \includegraphics[width = 0.23\linewidth, height = 0.22\linewidth]{Fig1.3.png}\\
        \vspace{0.3em}
        \includegraphics[width = 0.23\linewidth, height = 0.22\linewidth]{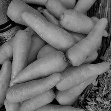} 
        \includegraphics[width = 0.23\linewidth, height = 0.22\linewidth]{Fig1.5.png} 
        \includegraphics[width = 0.23\linewidth, height = 0.22\linewidth]{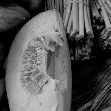}
        \includegraphics[width = 0.23\linewidth, height = 0.22\linewidth]{Fig1.7.png}\\ 
        \vspace{0.3em}
        \includegraphics[width = 0.23\linewidth, height = 0.22\linewidth]{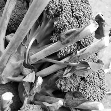} 
        \includegraphics[width = 0.23\linewidth, height = 0.22\linewidth]{Fig1.9.png} 
        \includegraphics[width = 0.23\linewidth, height = 0.22\linewidth]{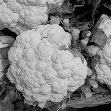}
        \includegraphics[width = 0.23\linewidth, height = 0.22\linewidth]{Fig1.11.png}\\ 
    \caption{\label{fig: Te-Im}
    Evaluation images
    }
    \end{subfigure}
    \caption{\label{fig: Tr-Te-Im}
   Training and evaluation images of the {\tt Kaggle} dataset.
    }
\end{figure}

\subsubsection{Important Parameter Settings}
For all solvers, the initial dictionary $D_0\in\mathbb{R}^{36\times256}$ is determined by employing a partial discrete cosine transform,  
the initial point value for sparse recovery is $x_0=\|D_i^Tp_i\|_{\infty}1_{256\times1}$, where $\lVert\cdot\rVert_{\infty}$ denotes infinity norm of a vector. 

We set $\mu=2^{-8}$ in the optimization problem \eqref{e:defprob1} in accordance to compressed sensing literature cf.~\cite{Ela}. 

The stopping condition for all solvers in line 7 of Algorithm \ref{alg:one} is given by \begin{equation}
  \|x_{k+1}-x_{k}\|_2\leq {\epsilon_1}\|x_k\|_2 
\end{equation}
with $\epsilon_1=10^{-7}$ for {\tt GSCG} and {\tt ISGA} and $\epsilon_1=10^{-5}$ for {\tt FISTA}, {\tt FPC-BB}, and {\tt TwIST} or if the number of iterations exceeds
100000.
Let us comment that by extensive computations not documented here, this difference in setting $\epsilon_1$ is necessary to achieve results of comparable quality.

For image reconstruction with learned dictionaries in terms of the original image ($\text{{Img}}_{\text{org}}$) and reconstructed image (${\text{Img}}_{\text{rec}}$), the stopping criterion is as
\begin{equation}\label{e:ReErFr0}
  \|\text{{Img}}_{\text{org}}-{\text{Img}}_{\text{rec}}\|_F\leq {\epsilon_2}\|\text{{Img}}_{\text{org}}\|_F 
\end{equation}
using the Frobenius norm, with $\epsilon_2=10^{-10}$ for {\tt GSCG} and {\tt ISGA} and $\epsilon_2=10^{-5}$ for {\tt FISTA}, {\tt FPC-BB}, and {\tt TwIST}, or if the number of iterations exceeds
5000000.
Again, let us comment that this difference in $\epsilon_2$ is necessary to achieve results of comparable quality.

All other parameters of any given solver are the default as described in the corresponding paper. 
\begin{figure}[H]
\centering
\begin{subfigure}[t]{0.49\textwidth}
    \centering
    \includegraphics[width=\textwidth]{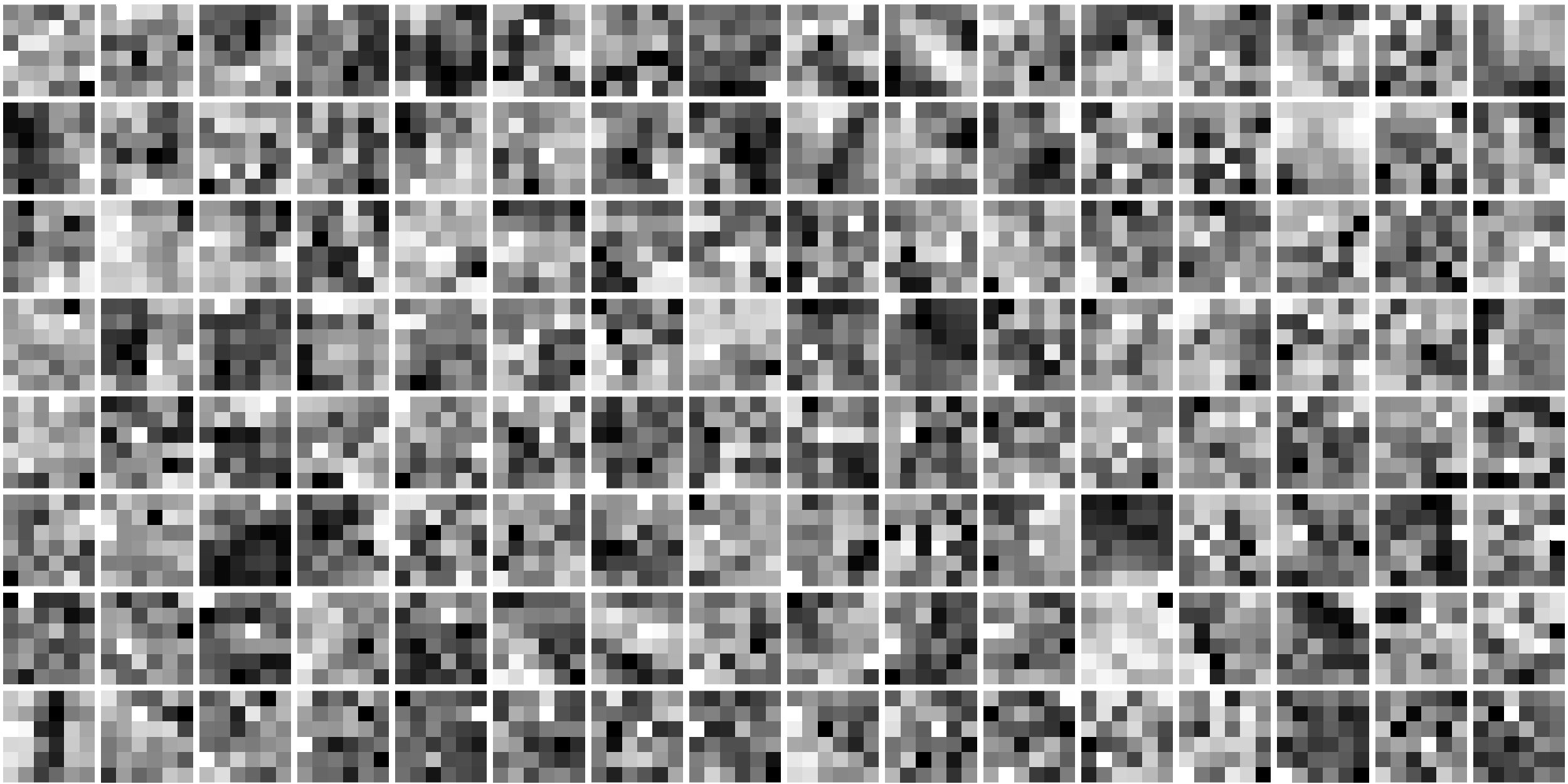}
    \caption{D12-\tt{FISTA}}
    \label{fig:D12Fis}
\end{subfigure}
\hfill
\begin{subfigure}[t]{0.49\textwidth}
    \centering
    \includegraphics[width=\textwidth]{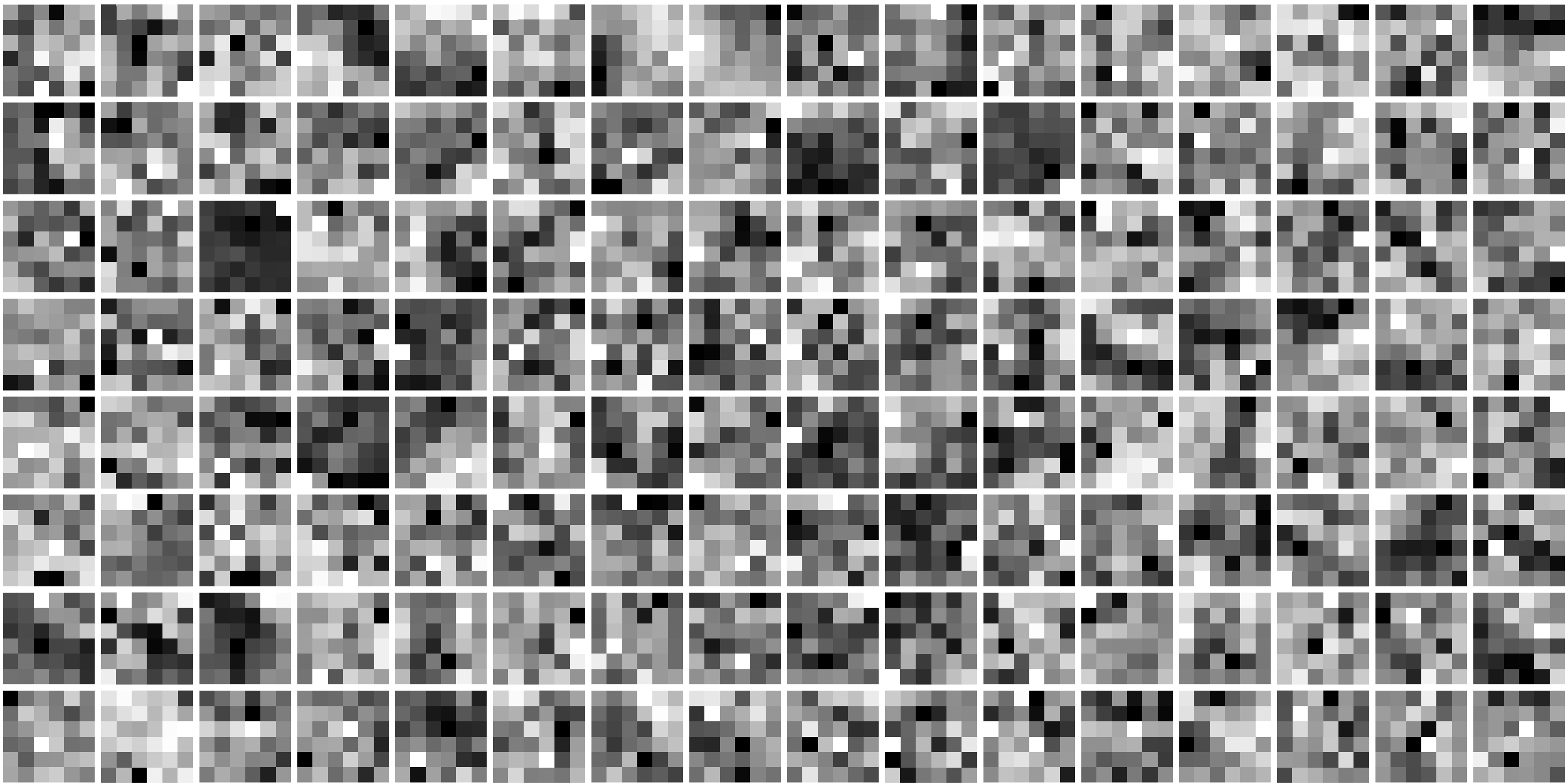}
    \caption{D12-\tt{FPC-BB}}
    \label{fig:D12FPCBB}
\end{subfigure}
\begin{subfigure}[t]{0.49\textwidth}
    \centering
    \includegraphics[width=\textwidth]{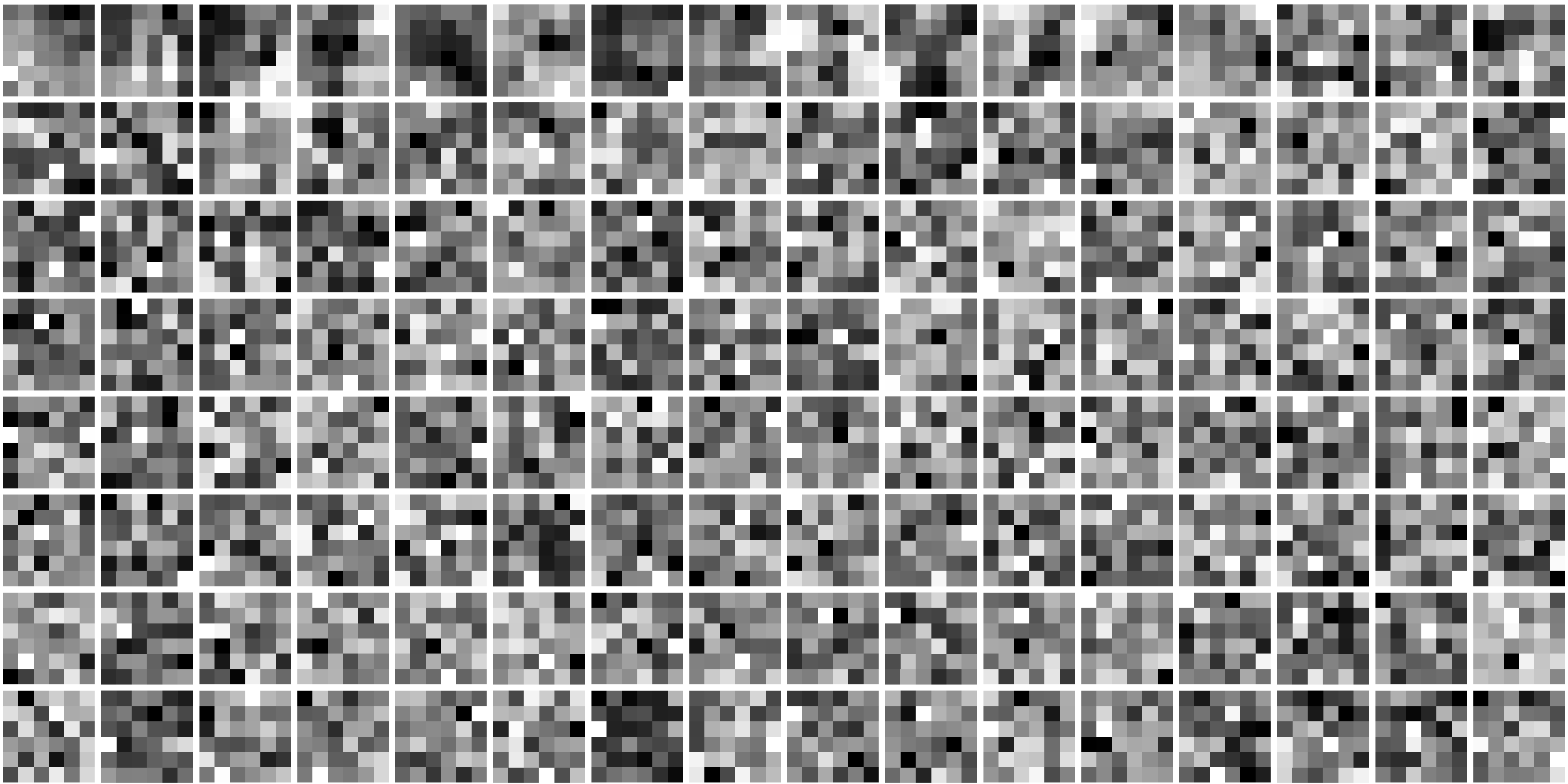}
    \caption{D12-\tt{TwIST}}
    \label{fig:D12TwIST}
\end{subfigure}
\hfill
\begin{subfigure}[t]{0.49\textwidth}
    \centering
    \includegraphics[width=\textwidth]{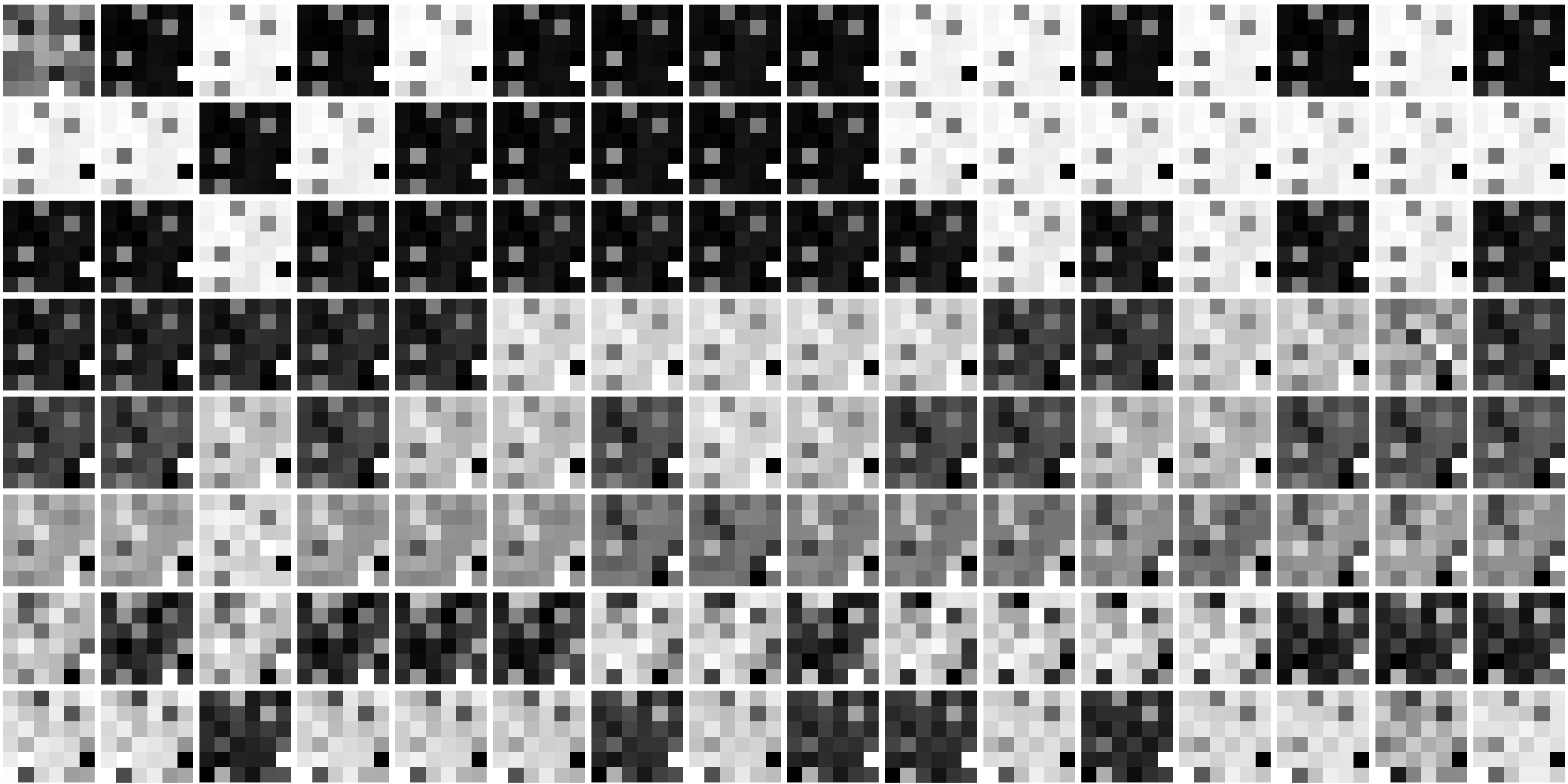}
    \caption{D12-\tt{GSCG}}
    \label{fig:D12GSCG}
\end{subfigure}
\vspace{0.5cm}
\begin{subfigure}[t]{0.49\textwidth}
    \centering
    \includegraphics[width=\textwidth]{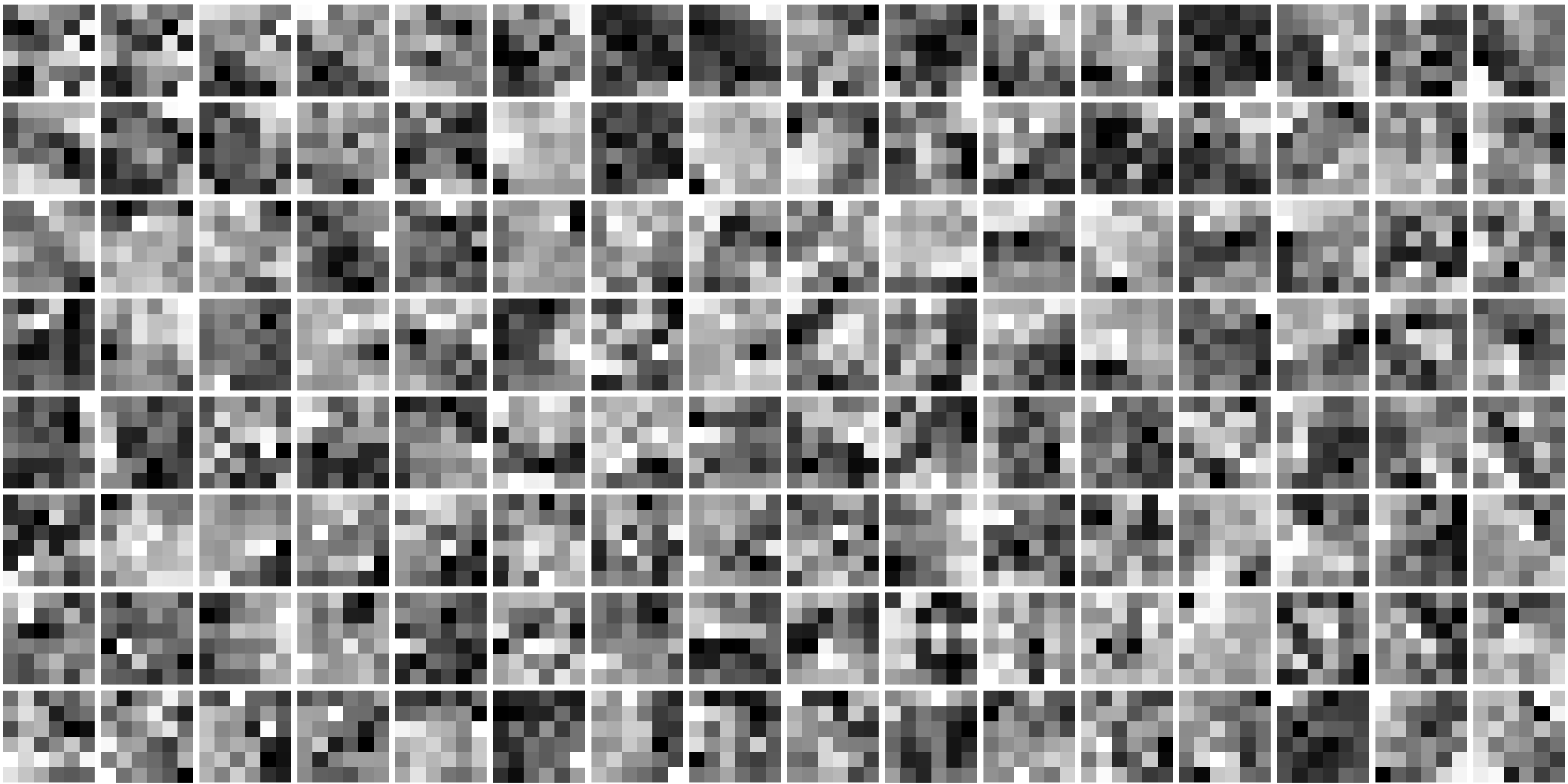}
    \caption{D12-\tt{ISGA}}
    \label{fig:D12ISGA}
\end{subfigure}
\caption{Partial visualization of the final learned dictionaries with {\tt FISTA}, {\tt FPC-BB}, {\tt TwIST}, {\tt GSCG}, and {\tt ISGA}.}
\label{fig:D12Comparison}
\end{figure}

\subsubsection{Learning the Dictionary}
Algorithms usually choose images from the data set randomly for the learning process. In this paper, to have a fair comparison between all sparse recovery solvers, we followed the given row-wise order in Figure \ref{fig: Tr-Im}. Given the initial dictionary $D_0$ with proper dimension, based on the ordered extracted patches of the first image, Algorithm \ref{alg:one} produces the first dictionary $\mathrm{D}1$. The method generates $\mathrm{D}2$ using the second training image with warm start $\mathrm{D}1$ at its first updating with respect to the first patch of the second image in Algorithm \ref{alg:two}. Continuing in this way for all optimization methods, we get the final learned dictionary called $\mathrm{D12}$. In Figure \ref{fig:D12Comparison}, we have a partial visualization of the learned dictionaries for all sparse recovery solvers. 

Let us make clear the notation $\mathrm{D1}, \ldots, \mathrm{D12}$.  $\mathrm{D1}$ is the dictionary generated by the patches of the first training image. $\mathrm{D2}$ is generated by the patches of the first and second training images and $\mathrm{D12}$ is the final dictionary generated by the whole training images. Here, $D_0=\mathrm{D}0$.   

\subsubsection{Image Recovery}
In the image reconstruction, we used evaluation images in Figure \ref{fig: Te-Im}. This set includes $6$ images from the training image dataset (columns $2$, $4$) as well as $6$ images not out of it but from the same categories (columns $1$, $3$). For the method reconstruction evaluations, we call images in odd columns and even columns as $\text{Test}_{i}$ and $\text{Train}_{i}$, respectively, for $i=1,\ldots,6$ and left to right. 

Reconstruction process is patch-wise in terms of the minimization problem \eqref{e:defprob1} for all sparse recovery algorithms. To reconstruct the whole image, the algorithms average the overlapped areas of the reconstructed patches based on the overlapping percentage, see Figure \ref{fig:gray-grid}.  
\begin{figure}[t]
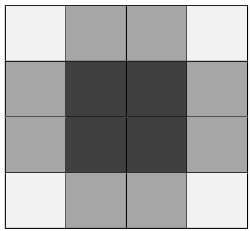

\centering
\renewcommand{\arraystretch}{1.88} 
\setlength{\tabcolsep}{2pt} 
\begin{tabular}{|>{\centering\arraybackslash}m{2em}|>{\centering\arraybackslash}m{2em}|>{\centering\arraybackslash}m{2em}|>{\centering\arraybackslash}m{2em}|}
\hline
\cellcolor{gray!10} & \cellcolor{gray!70} & \cellcolor{gray!70} & \cellcolor{gray!10} \\ \hline
\cellcolor{gray!70} & \cellcolor{gray!150} & \cellcolor{gray!150} & \cellcolor{gray!70} \\ \hline
\cellcolor{gray!70} & \cellcolor{gray!150} & \cellcolor{gray!150} & \cellcolor{gray!70} \\ \hline
\cellcolor{gray!10} & \cellcolor{gray!70} & \cellcolor{gray!70} & \cellcolor{gray!10} \\ \hline
\end{tabular}
\caption{Schematic grid to show patch averaging}
\label{fig:gray-grid}
\end{figure}

\subsubsection{Evaluation of Methods for Image Recovery}
In this section, we investigate the reconstruction process of all sparse solvers. Figure \ref{fig:first_page} and Figure \ref{fig:second_page} show the diagrams of absolute error  ({\tt AbEr}) in the left column and relative error ({\tt ReEr}) in the right column for the image reconstruction of evaluation dataset, Figure \ref{fig: Te-Im}, for the initial dictionary $\mathrm{D}0$ and all learned dictionaries $\mathrm{D}1$-$\mathrm{D12}$. Here, our error measurements are
\begin{equation}\label{e:AbErFr}
 {\tt AbEr}=\|\text{{Img}}_{\text{org}}-{\text{Img}}_{\text{rec}}\|_F
\quad \textrm{and} \quad
 {\tt ReEr}=\frac{\|\text{{Img}}_{\text{org}}-{\text{Img}}_{\text{rec}}\|_F}{\|\text{{Img}}_{\text{org}}\|_F}
\end{equation}

As can be seen, {\tt FISTA}, {\tt ISGA} and also {\tt FPC-BB} give a very similar behaviour in terms of error measures. 
Also they show an expected behaviour as errors decrease when putting in more data and thus expanding the dictionaries.
Concerning this expected behaviour, let us comment at this point also on what can actually be observed in the Figures \ref{fig:first_page} and \ref{fig:second_page}.
As pointed out before, the images in our image datasets are denoted as $\text{Test}_{i}$ and $\text{Train}_{i}$, respectively, for $i=1,\ldots,6$.
Correspondingly, each graph in these figures gives an account of the reconstruction of the corresponding individual image, given the indicated dictionary.
What is thus desirable is an error decrease that is ideally uniformly over all the graphs, meaning for all individual reconstructed images.

In contrast, {\tt TwIST} shows a very rapid decrease in errors, but on the other hand it seems that the process of adding data during training has no beneficial effect, which is surely not intuitive and may make {\tt TwIST} not a potentially attractive method of choice if learning over a large dataset is possible.
The {\tt GSCG} 
\begin{figure}[H]
\centering
\begin{subfigure}{0.49\textwidth}
\centering
\includegraphics[width=\textwidth]{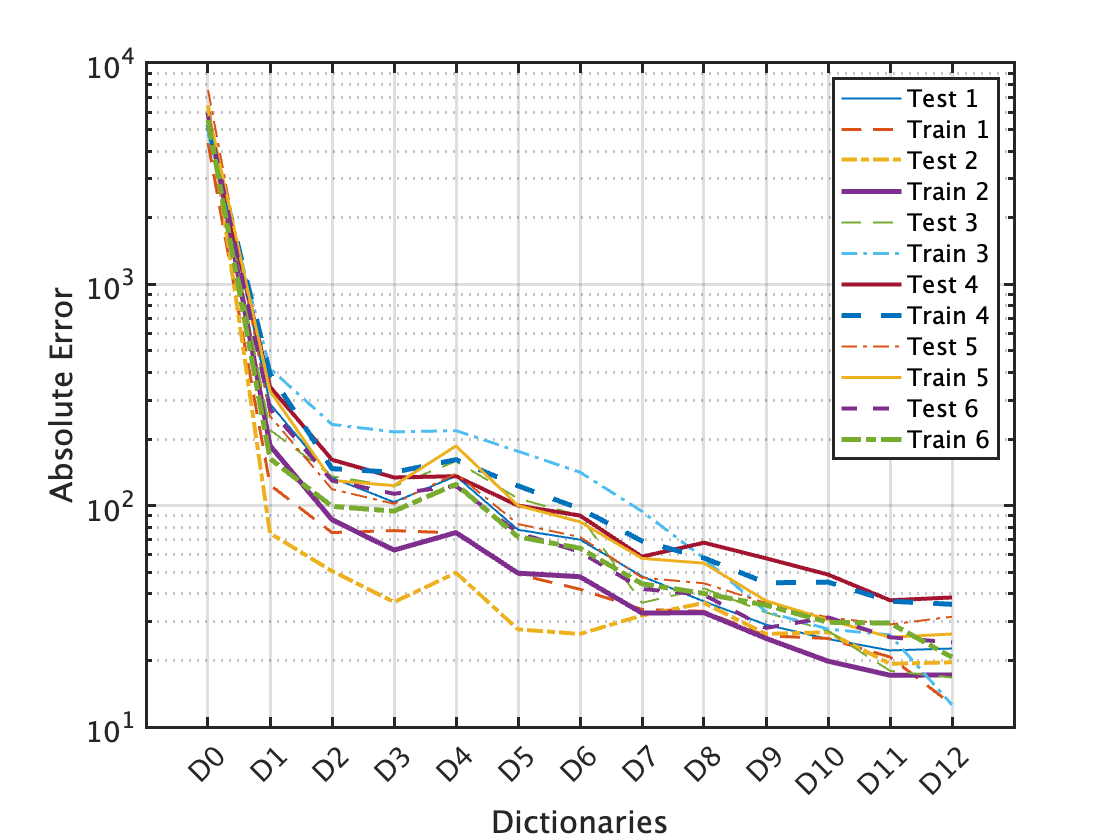}
\caption{{\tt AbEr}-\tt{FISTA}}
\label{fig:abs_error_fista}
\end{subfigure}
\hfill
\begin{subfigure}{0.49\textwidth}
\centering
\includegraphics[width=\textwidth]{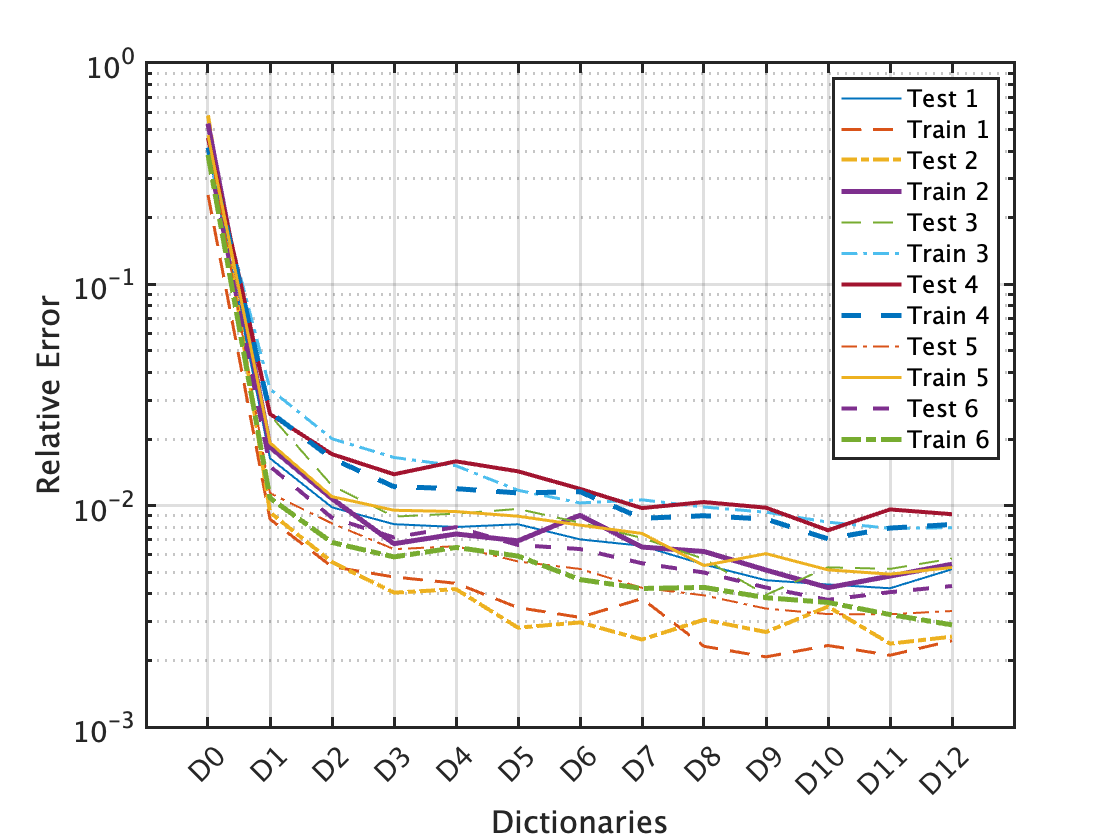}
\caption{{\tt ReEr}-\tt{FISTA}}
\label{fig:rel_error_fista}
\end{subfigure}

\vspace{0.5cm}

\begin{subfigure}{0.49\textwidth}
\centering
\includegraphics[width=\textwidth]{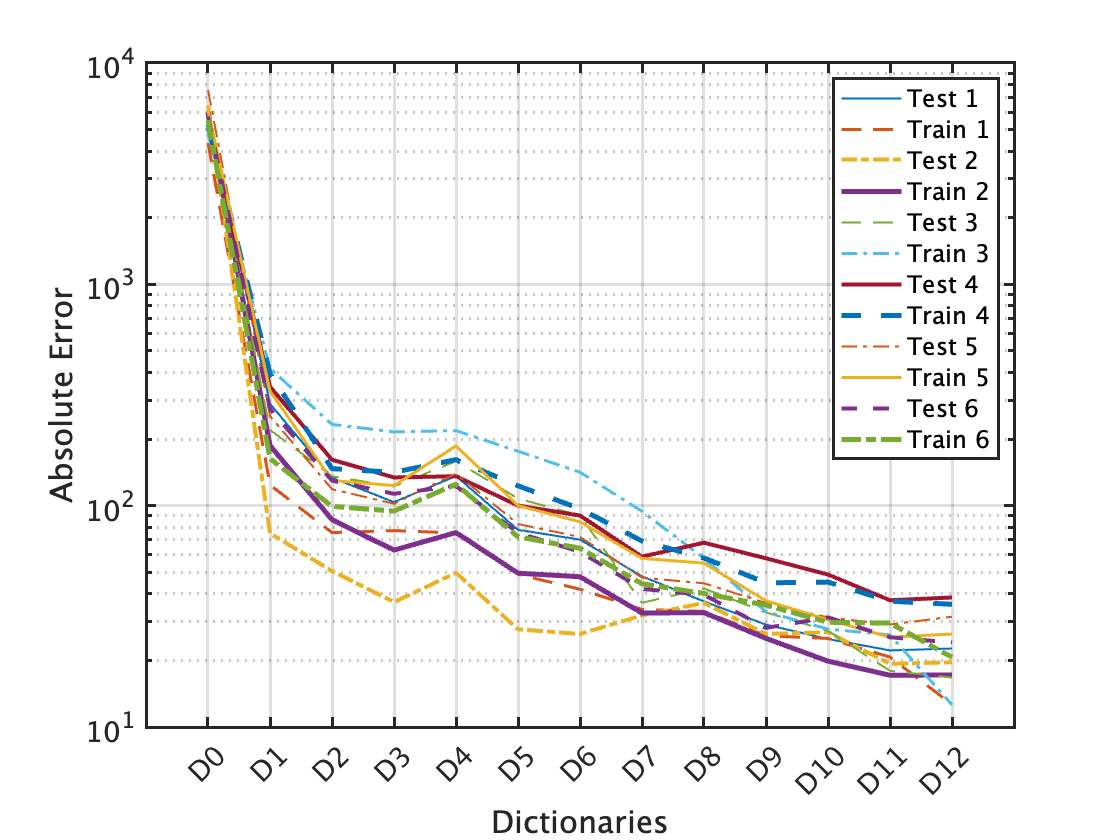}
\caption{{\tt AbEr}-\tt{ISGA}}
\label{fig:abs_error_isga}
\end{subfigure}
\hfill
\begin{subfigure}{0.49\textwidth}
\centering
\includegraphics[width=\textwidth]{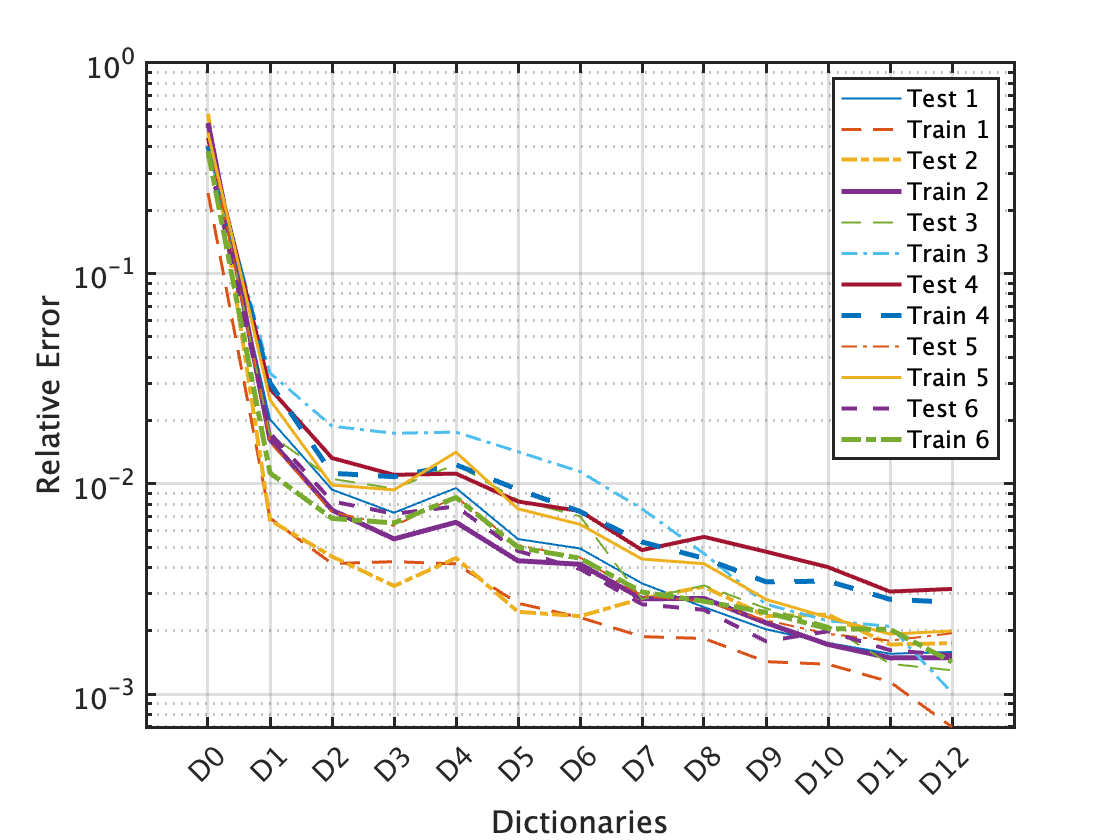}
\caption{{\tt ReEr}-\tt{ISGA}}
\label{fig:rel_error_isga}
\end{subfigure}

\vspace{0.5cm}

\begin{subfigure}{0.49\textwidth}
\centering
\includegraphics[width=\textwidth]{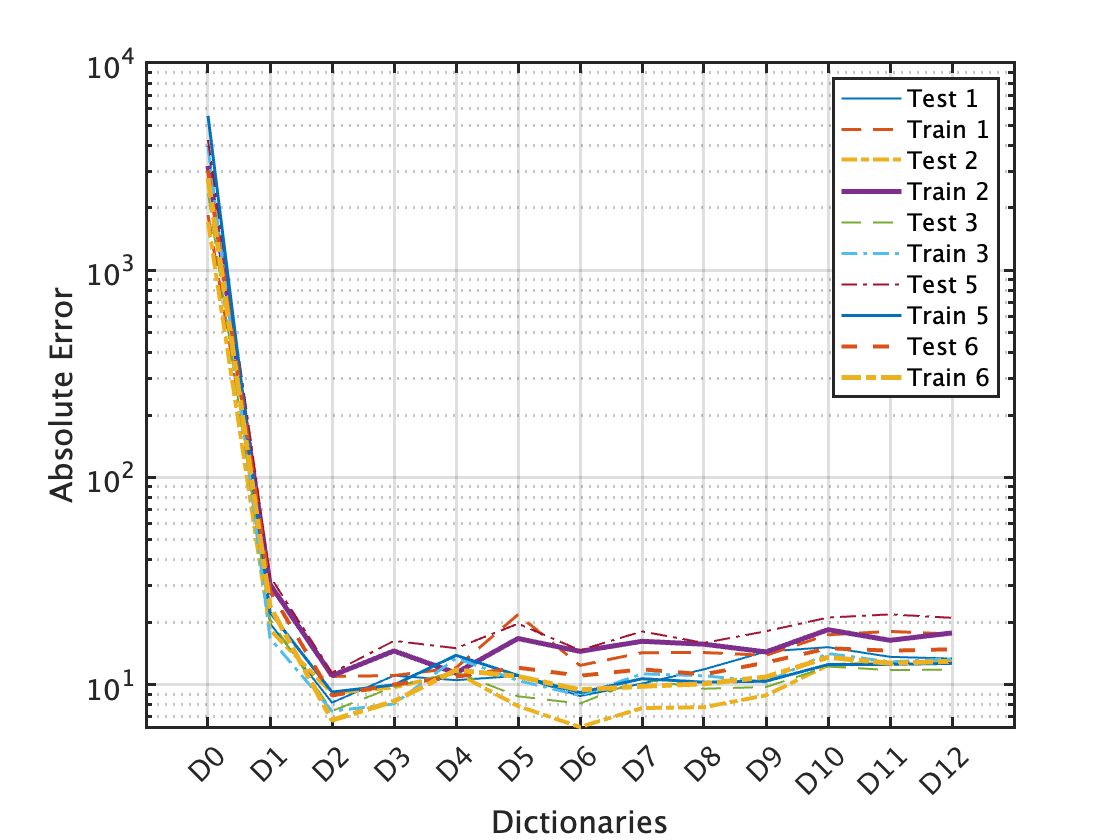}
\caption{{\tt AbEr}-\tt{TwIST}}
\label{fig:abs_error_twist}
\end{subfigure}
\hfill
\begin{subfigure}{0.49\textwidth}
\centering
\includegraphics[width=\textwidth]{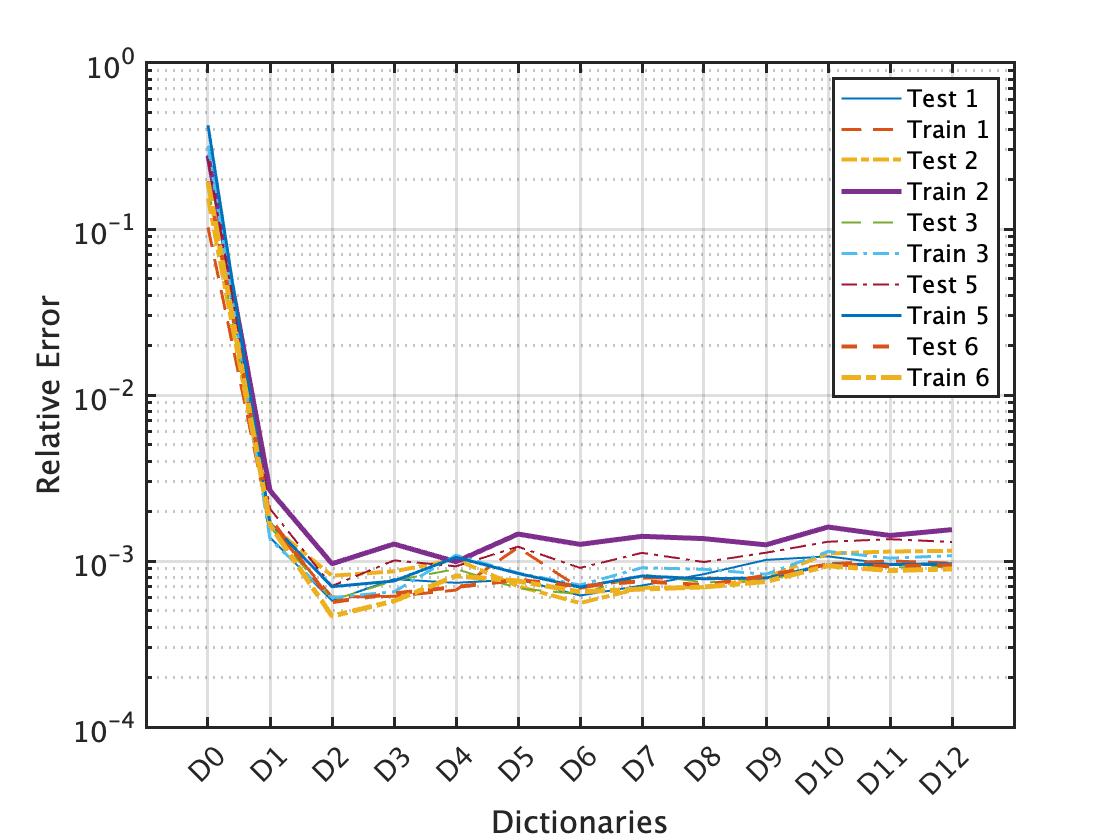}
\caption{{\tt ReEr}-\tt{TwIST}}
\label{fig:rel_error_twist}
\end{subfigure}

\caption{Diagram of absolute reconstruction errors
versus learned dictionaries (left column) and relative reconstruction errors versus learned dictionaries (right column) for {\tt FISTA}, {\tt ISGA}, and {\tt TwIST}. Note that we decided to show here the natural (not unified) scales of results since we would like to emphasize the general behaviour of errors as they develop when expanding the dictionaries.}
\label{fig:first_page}
\end{figure}

\begin{figure}[h]
\centering

\begin{subfigure}{0.49\textwidth}
\centering
\includegraphics[width=\textwidth]{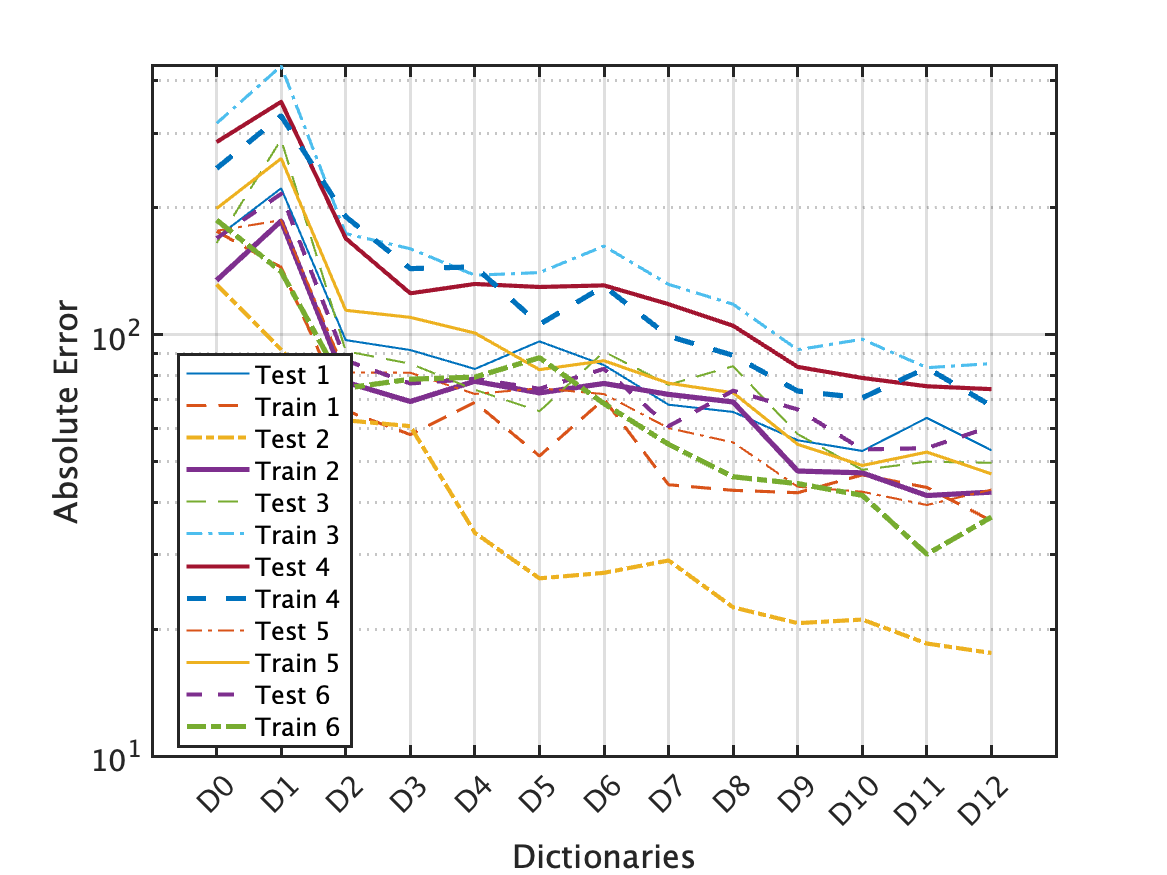}
\caption{{\tt AbEr}-\tt{FPC-BB}}
\label{fig:abs_error_fpcbb}
\end{subfigure}
\hfill
\begin{subfigure}{0.49\textwidth}
\centering
\includegraphics[width=\textwidth]{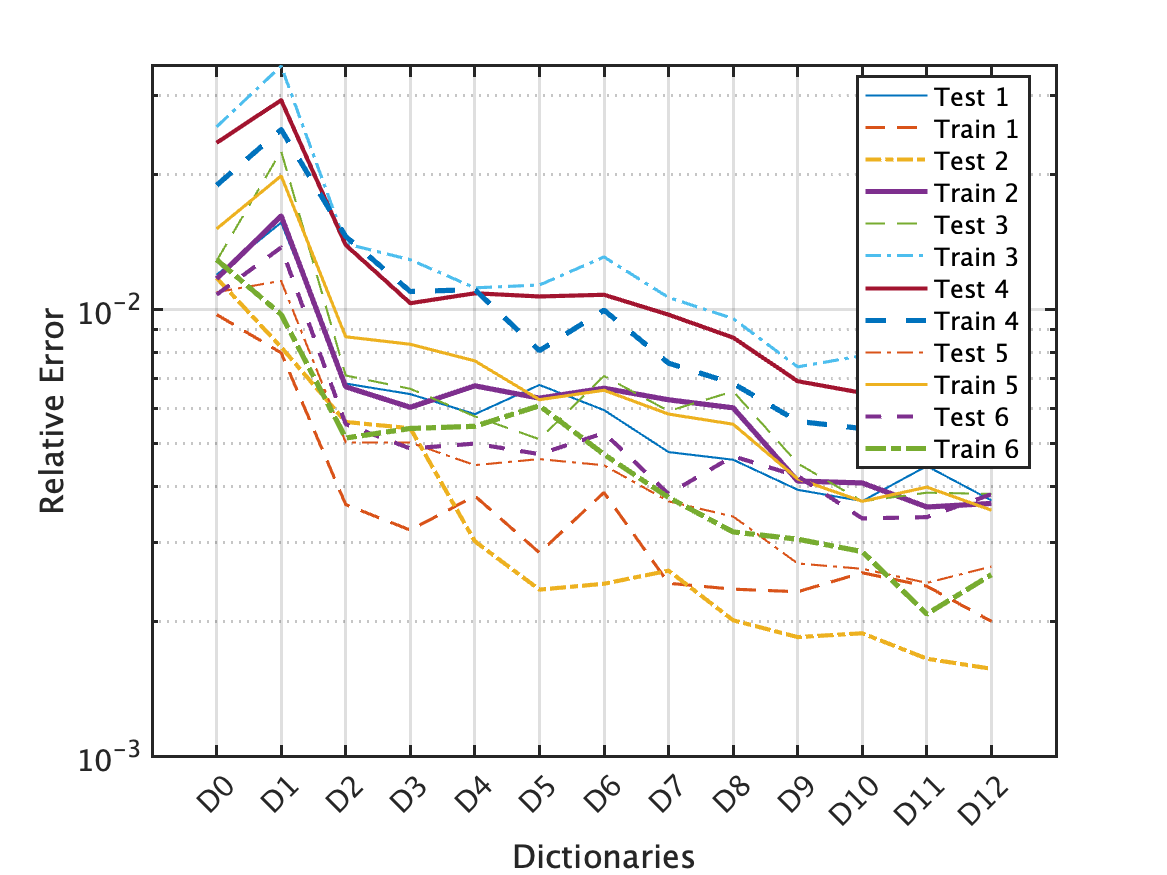}
\caption{{\tt ReEr}-\tt{FPC-BB}}
\label{fig:rel_error_fpcbb}
\end{subfigure}

\vspace{0.3cm}

\begin{subfigure}{0.49\textwidth}
\centering
\includegraphics[width=\textwidth]{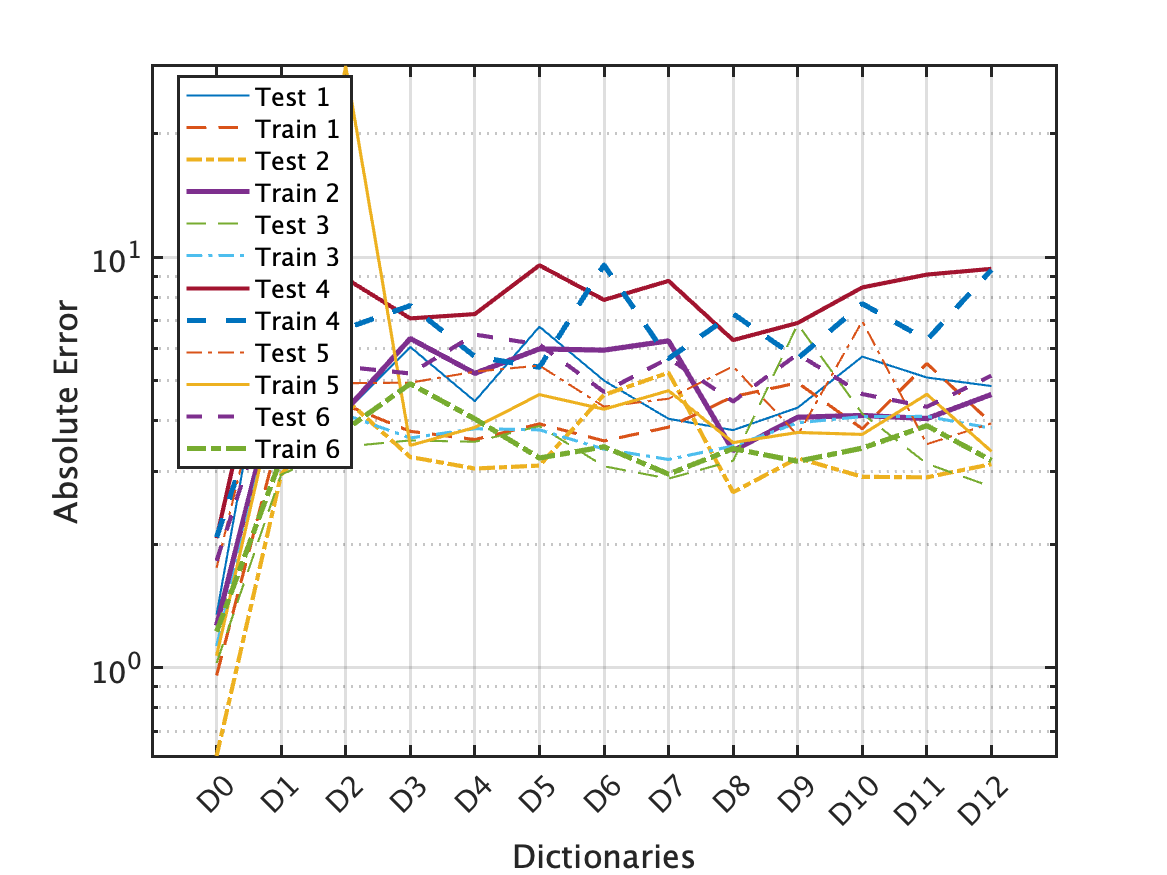}
\caption{{\tt AbEr}-\tt{GSCG}}
\label{fig:abs_error_gscg}
\end{subfigure}
\hfill
\begin{subfigure}{0.49\textwidth}
\centering
\includegraphics[width=\textwidth]{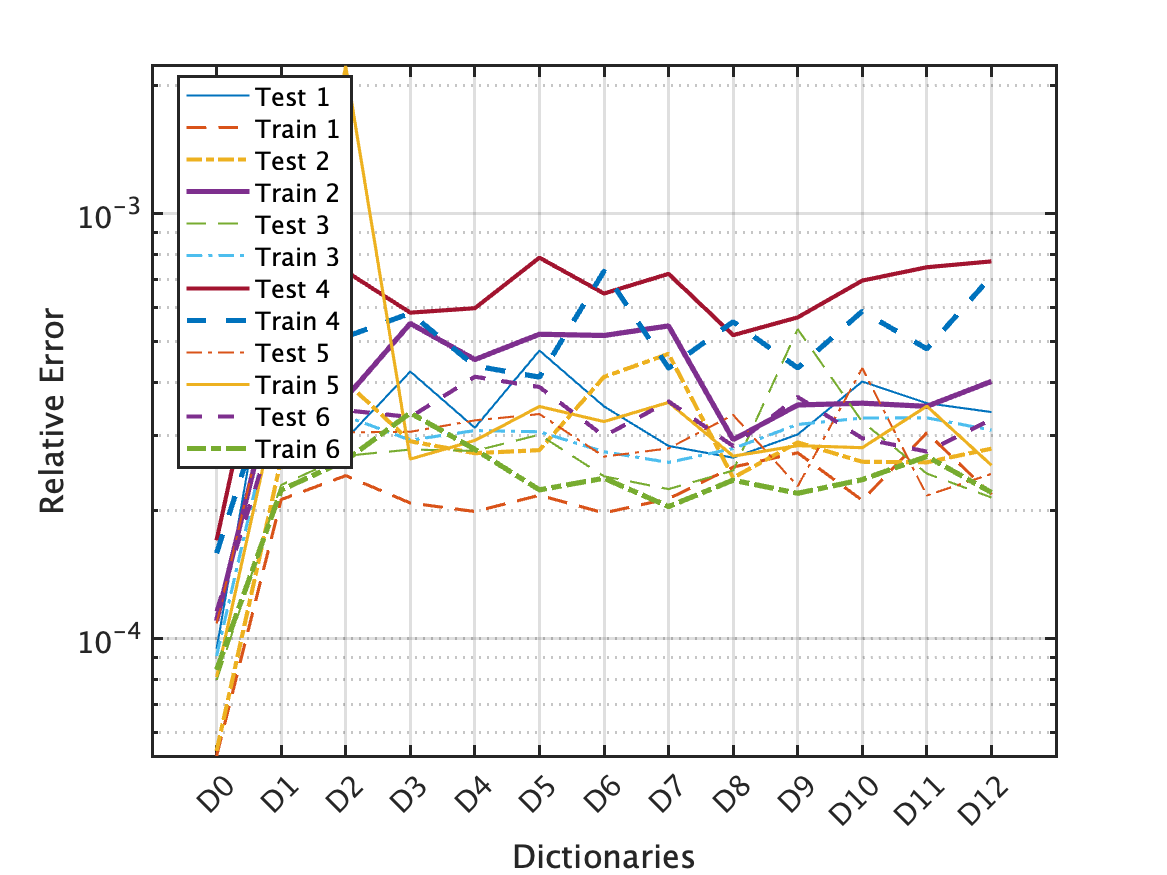}
\caption{{\tt ReEr}-\tt{GSCG}}
\label{fig:rel_error_gscg}
\end{subfigure}

\caption{Diagram of absolute reconstruction errors
versus learned dictionaries (left column) and relative reconstruction errors versus learned dictionaries (right column) for {\tt FPC-BB} and {\tt GSCG}.
Note that we decided to show here the natural (not unified) scales of results since we would like to emphasize the general behaviour of errors as they develop when expanding the dictionaries.
}
\label{fig:second_page}
\end{figure}
method appears to enable very low reconstruction errors even with little training data for dictionary construction, which is different from the behaviour of {\tt TwIST}. Like for {\tt TwIST}, it appears not beneficial to expand the training database.

For quantitative evaluation, we may also consider averages computed over the whole available dataset, as this gives a reasonable statistical account.
To this end, we consider in addition to the error measures defined in \eqref{e:AbErFr} the infinity norm of differences. This captures the highest deviations over complete images, which complements norms employing a summation of differences over imagery:
\begin{equation}\label{e:InErFr}
 {\tt InEr}=\|\text{{Img}}_{\text{org}}-{\text{Img}}_{\text{rec}}\|_{\infty}
\end{equation}
The results on averaged error measures together with the infinity norm of errors are summarized in Table \ref{table:average-errors}.

\begin{table}[H]
\renewcommand{\arraystretch}{1.2} 
\setlength{\tabcolsep}{3pt}       
\centering

\begin{subtable}[t]{\textwidth}
\centering
\begin{tabular}{|c||>{\centering\arraybackslash}p{1cm}|>{\centering\arraybackslash}p{1cm}|>{\centering\arraybackslash}p{0.85cm}||>{\centering\arraybackslash}p{1cm}|>{\centering\arraybackslash}p{0.85cm}|>{\centering\arraybackslash}p{0.65cm}||>{\centering\arraybackslash}p{1cm}|>{\centering\arraybackslash}p{0.85cm}|>{\centering\arraybackslash}p{0.65cm}|}
\hline
\textbf{Method} & \multicolumn{3}{|c||}{$\mathrm{D}0$} & \multicolumn{3}{|c||}{$\mathrm{D}1$} & \multicolumn{3}{|c|}{$\mathrm{D}2$} \\
\cline{2-10}
& {\tt ReEr} & {\tt AbEr} & {\tt InEr} & {\tt ReEr} & {\tt AbEr} & {\tt InEr} & {\tt ReEr} & {\tt AbEr} & {\tt InEr} \\
\hline
\tt{FISTA} & 0.43237 & 5824.15 & 166.66 & 0.01841 & 245.33 & 8.10 & 0.01101 & 147.15 & 5.13 \\
\hline
\tt{FPC-BB} & 0.01464 & 196.58 & 8.77 & 0.01792 & 238.35 & 7.70 & 0.00806 & 106.99 & 3.47 \\
\hline
\tt{TwIST} & 0.19906 & 2726.35 & 118.97 & 0.00146 & 20.34 & 0.71 & 0.00055 & 7.54 & 0.35 \\
\hline
\tt{GSCG} & 0.00010 & 1.36 & 0.04 & 0.00042 & 5.71 & 0.25 & 0.00052 & 6.99 & 0.40 \\
\hline
\tt{ISGA} & 0.42500 & 5723.25 & 165.73 & 0.01900 & 255.92 & 9.81 & 0.00931 & 125.05 & 4.29 \\
\hline
\end{tabular}
\end{subtable}

\vspace{0.3cm}

\begin{subtable}[t]{\textwidth}
\centering
\begin{tabular}{|c||>{\centering\arraybackslash}p{1cm}|>{\centering\arraybackslash}p{1cm}|>{\centering\arraybackslash}p{0.85cm}||>{\centering\arraybackslash}p{1cm}|>{\centering\arraybackslash}p{0.85cm}|>{\centering\arraybackslash}p{0.65cm}||>{\centering\arraybackslash}p{1cm}|>{\centering\arraybackslash}p{0.85cm}|>{\centering\arraybackslash}p{0.65cm}|}
\hline
\textbf{Method} & \multicolumn{3}{|c||}{$\mathrm{D}3$} & \multicolumn{3}{|c||}{$\mathrm{D}4$} & \multicolumn{3}{|c|}{$\mathrm{D}5$} \\
\cline{2-10}
& {\tt ReEr} & {\tt AbEr} & {\tt InEr} & {\tt ReEr} & {\tt AbEr} & {\tt InEr} & {\tt ReEr} & {\tt AbEr} & {\tt InEr} \\
\hline
\tt{FISTA} & 0.00868 & 116.61 & 3.80 & 0.00889 & 119.26 & 4.13 & 0.00797 & 106.85 & 3.46 \\
\hline
\tt{FPC-BB} & 0.00712 & 94.94 & 3.05 & 0.00674 & 90.24 & 2.98 & 0.00627 & 83.99 & 2.85 \\
\hline
\tt{TwIST} & 0.00065 & 9.07 & 0.39 & 0.00074 & 10.16 & 0.47 & 0.00077 & 10.88 & 0.48 \\
\hline
\tt{GSCG} & 0.00037 & 4.99 & 0.26 & 0.00035 & 4.69 & 0.32 & 0.00038 & 5.16 & 0.34 \\
\hline
\tt{ISGA} & 0.00820 & 110.56 & 4.03 & 0.00980 & 132.05 & 5.77 & 0.00650 & 86.76 & 3.19 \\
\hline
\end{tabular}
\end{subtable}

\vspace{0.3cm}

\begin{subtable}[t]{\textwidth}
\centering
\begin{tabular}{|c||>{\centering\arraybackslash}p{1cm}|>{\centering\arraybackslash}p{1cm}|>{\centering\arraybackslash}p{0.85cm}||>{\centering\arraybackslash}p{1cm}|>{\centering\arraybackslash}p{0.85cm}|>{\centering\arraybackslash}p{0.65cm}||>{\centering\arraybackslash}p{1cm}|>{\centering\arraybackslash}p{0.85cm}|>{\centering\arraybackslash}p{0.65cm}|}
\hline
\textbf{Method} & \multicolumn{3}{|c||}{$\mathrm{D}6$} & \multicolumn{3}{|c||}{$\mathrm{D}7$} & \multicolumn{3}{|c|}{$\mathrm{D}8$} \\
\cline{2-10}
& {\tt ReEr} & {\tt AbEr} & {\tt InEr} & {\tt ReEr} & {\tt AbEr} & {\tt InEr} & {\tt ReEr} & {\tt AbEr} & {\tt InEr} \\
\hline
\tt{FISTA} & 0.00738 & 98.61 & 3.26 & 0.00643 & 86.37 & 2.80 & 0.00588 & 78.24 & 2.39 \\
\hline
\tt{FPC-BB} & 0.00674 & 90.26 & 2.87 & 0.00559 & 74.18 & 2.32 & 0.00528 & 70.33 & 2.11 \\
\hline
\tt{TwIST} & 0.00062 & 8.61 & 0.45 & 0.00073 & 10.06 & 0.48 & 0.00071 & 9.81 & 0.49 \\
\hline
\tt{GSCG} & 0.00038 & 4.99 & 0.28 & 0.00036 & 4.82 & 0.23 & 0.00031 & 4.29 & 0.34 \\
\hline
\tt{ISGA} & 0.00552 & 73.99 & 2.74 & 0.00371 & 49.74 & 1.76 & 0.00339 & 45.39 & 1.64 \\
\hline
\end{tabular}
\end{subtable}

\vspace{0.3cm}

\begin{subtable}[t]{\textwidth}
\centering
\begin{tabular}{|c||>{\centering\arraybackslash}p{1cm}|>{\centering\arraybackslash}p{1cm}|>{\centering\arraybackslash}p{0.85cm}||>{\centering\arraybackslash}p{1cm}|>{\centering\arraybackslash}p{0.85cm}|>{\centering\arraybackslash}p{0.65cm}||>{\centering\arraybackslash}p{1cm}|>{\centering\arraybackslash}p{0.85cm}|>{\centering\arraybackslash}p{0.65cm}|}
\hline
\textbf{Method} & \multicolumn{3}{|c||}{$\mathrm{D}9$} & \multicolumn{3}{|c||}{$\mathrm{D10}$} & \multicolumn{3}{|c|}{$\mathrm{D11}$} \\
\cline{2-10}
& {\tt ReEr} & {\tt AbEr} & {\tt InEr} & {\tt ReEr} & {\tt AbEr} & {\tt InEr} & {\tt ReEr} & {\tt AbEr} & {\tt InEr} \\
\hline
\tt{FISTA} & 0.00532 & 70.76 & 2.16 & 0.00489 & 65.29 & 2.07 & 0.00496 & 66.05 & 2.06 \\
\hline
\tt{FPC-BB} & 0.00424 & 56.90 & 1.79 & 0.00403 & 54.04 & 1.65 & 0.00394 & 52.89 & 1.63 \\
\hline
\tt{TwIST} & 0.00074 & 10.32 & 0.42 & 0.00092 & 12.66 & 0.51 & 0.00088 & 12.26 & 0.47 \\
\hline
\tt{GSCG} & 0.00035 & 4.70 & 0.39 & 0.00037 & 4.98 & 0.28 & 0.00035 & 4.71 & 0.33 \\
\hline
\tt{ISGA} & 0.00256 & 34.34 & 1.33 & 0.00228 & 30.70 & 1.15 & 0.00189 & 25.59 & 1.06 \\
\hline
\end{tabular}
\end{subtable}

\vspace{0.3cm}

\begin{subtable}[t]{\textwidth}
\centering
\begin{tabular}{|c||>{\centering\arraybackslash}p{1cm}|>{\centering\arraybackslash}p{1cm}|>{\centering\arraybackslash}p{0.85cm}|}
\hline
\textbf{Method} & \multicolumn{3}{|c|}{$\mathrm{D12}$} \\
\cline{2-4}
& {\tt ReEr} & {\tt AbEr} & {\tt InEr}  \\
\hline
\tt{FISTA} & 0.00521 & 69.46 & 2.12 \\
\hline
\tt{FPC-BB} & 0.00380 & 51.07 & 1.70 \\
\hline
\tt{TwIST} & 0.00089 & 12.36 & 0.52 \\
\hline
\tt{GSCG} & 0.00036 & 4.79 & 0.30 \\
\hline
\tt{ISGA} & 0.00172 & 23.19 & 1.06 \\
\hline
\end{tabular}
\end{subtable}
\caption{Average values of relative error ({\tt ReEr}), absolute error ({\tt AbEr}), and infinity error ({\tt InEr}) over all evaluation images, Figure \ref{fig: Te-Im}, for initial dictionary ($\mathrm{D}0$) and all learned dictionaries ($\mathrm{D}1$-$\mathrm{D}12$).
\label{table:average-errors}}
\end{table}

The numbers of the table give a more statistical account of the behaviour observed in the graphs. 
Let us note that for these computations, images are considered as entities where the tonal range is given by standard grey value integers in $[0, \ldots,255]$.
This implies for instance, that an infinity norm error of one means that there is a difference of one grey level.

For {\tt FISTA} and {\tt ISGA}, the errors start very high but are reduced quickly and almost monotonically when increasing the dictionary. Interestingly, {\tt ISGA} gives in comparison to {\tt FISTA} consistently lower errors as more data is supplemented, while {\tt FISTA} seems to struggle in lowering errors in comparison to other methods. For the {\tt FPC-BB} method, errors are lower in comparison with {\tt FISTA} and errors are lowered almost monotonically. The more data is supplemented by dictionaries, {\tt ISGA} appears to be advantageous compared to {\tt FPC-BB}.

The methods {\tt TwIST} and {\tt GSCG} enable very low errors that stay approximately constant even when extending dictionaries, with some fluctuations after construction of second or third dictionary. 
While {\tt GSCG} shows in comparison a lower error.

We supplement this quantitative evaluation of reconstruction quality and dictionaries by an assessment of computation times to gain an understanding of overall computational efficiency of the methods, see Table \ref{tab:cpu_time}.
Since dictionary learning takes place over patches, we propose to evaluate over patches plus giving an account for the complete images.

We observe that {\tt ISGA} is by far the fastest method.
Comparing with {\tt FISTA} and {\tt FPC-BB} which result in similar reconstruction quality the difference is quite apparent in total.
As for {\tt TwIST} and {\tt GSCG}, the {\tt TwIST}
method appears to be faster, whereas we recall that {\tt GSCG} gave lower error measures.
Comparing then {\tt ISGA} with {\tt TwIST} and {\tt GSCG}, the {\tt ISGA} method is faster but there is relatively high difference in error measures. 
Which may be considered negligible when using {\tt ISGA} together with larger amounts of training data.

\begin{table}[h]
\centering
\caption{Average of CPU time in the image reconstruction  for a single patch and a single image without averaging process over all evaluation images, Figure \ref{fig: Te-Im}, and dictionaries ($\mathrm{D}0$-$\mathrm{D}12$) for each method (in seconds).}
\label{tab:cpu_time}
\begin{tabular}{|c||c||c|}
\hline
\textbf{Method} & \textbf{CPU(Sec) - Per single patch} & \textbf{CPU(Sec) - Per single image}\\
\hline
FISTA  & 0.515105 & 667.5761\\
\hline
FPC-BB & 0.042173 & 54.6562\\
\hline
TwIST  & 0.028969 & 37.5438\\
\hline
GSCG   & 0.097154 & 125.9116\\
\hline
ISGA   & 0.004019 & 5.2086\\

\hline
\end{tabular}
\end{table}

\section{Summary and Conclusion}

We have presented in this paper a comprehensive study of optimization methods based on iterative shrinkage for the purpose of online SDL. 
We considered these methods as they represent a class of state-of-the-art methods for the purpose.
As this is not a study we have encountered in a similar way within the literature, we have proposed some means of evaluation that may represent a novelty in the field of dictionary learning and SDL optimization.

We conjecture that the error measures together with the construction of the datasets and dictionaries gave some interesting and relevant insights about working with the sparse coding step in SDL and possible ways to apply the considered optimization methods. For example, this may depend on availability of useful training data, related to observation of different reconstruction quality for {\tt FISTA/FPC-BB/ISGA} in contrast to {\tt TwIST/GSCG}. 

For future work, it appears that it is desirable to consider varying parameters for the underlying basis pursuit denoising problem, and to study the varying degree of the actual sparsity that can be achieved by such methods. 
Let us also note that in the current paper we did not explore fully the online nature of the chosen SDL algorithm.

\section{Appendix}
In this part, we give more details about the equations (\ref{e:algeq1}) and \ref{e:algeq2}) in Algorithm \ref{alg:one} and Algorithm \ref{alg:two}, respectively.
First, we recall some properties of the trace and the matrix derivative. The derivative of a matrix is usually referred as the gradient, denoted as $\nabla$. 
\begin{proposition}\label{pro1}
Let $A$, $B$, and $C$ be matrices with proper dimension, then 
\begin{itemize}
    \item[P1.] $\text{Tr}(A)=\text{Tr}(A^T)$ 
    \vspace{2.mm}
    \item[P2.] $\text{Tr}(ABC)=\text{Tr}(BCA)=\text{Tr}(CAB)$ 
    \vspace{2.mm}
    \item[P3.]$\nabla\text{Tr}(A)=\text{Tr}(\nabla A)$ 
    \vspace{2.mm}
    \item[P4.] $\nabla_A\text{Tr}(AB)=B^T$
    \vspace{2.mm}
    \item[P5.] $\nabla_A\text{Tr}(ABA^TC)=CAB+C^TAB^T$
\end{itemize}
\end{proposition}
Turning to (\ref{e:algeq1}), we have minimization in terms of $D$. Rewriting the cost function 
\begin{equation*}
\begin{split}
 \frac{1}{2}\|Dx_i-p_i\|_2^2+\mu\|x_i\|_1&=\frac{1}{2}(Dx_i-p_i)^T(Dx_i-p_i)+\mu\|x_i\|_1\\&=\frac{1}{2}(x_i^TD^TDx_i-x_i^TDp_i-p_i^TDx_i+\|p_i\|^2)+\mu\|x_i\|_1   
\end{split}
\end{equation*}
and concerning the terms with $D$, we have
\begin{equation*}
D_k=\underset{D\in C}{\textrm{\argmin}}\frac{1}{k}\sum_{i=1}^{k} \frac{1}{2}\Big(x_i^TD^TDx_i-\big(x_i^TD^Tp_i+p_i^TDx_i\big)\Big)
\end{equation*}
Now, we show that  
\begin{itemize}
    \item [i.] $\text{Tr}(D^TDA_k)=\sum_{i=1}^kx_i^TD^TDx_i$
    \item[ii.] $\text{Tr}(D^TB_k)=\sum_{i=1}^k\frac{1}{2}\big(x_i^TD^Tp_i+p_i^TDx_i\big)$
\end{itemize}

\begin{proof}
By the definition $A_k=\sum_{i=1}^kx_ix_i^T$ and $B_k=\sum_{i=1}^kp_ix_i^T$. Using properties \textit{P1} and \textit{P2} of Proposition \ref{pro1} and the fact that the trace of the outer product of two real vectors is equivalent to the inner product of them, we get  
    \begin{itemize}
    \item[i.]\begin{equation*}
\text{Tr}(D^TDA_k)=\text{Tr}\big(\sum_{i=1}^k D^TDx_ix_i^T\big)
=\sum_{i=1}^k\text{Tr}\big( {\underbrace{Dx_i}_{v}}{\underbrace{x_i^TD^T}_{v^T}}\big)
=\sum_{i=1}^k{\underbrace{x_i^TD^TDx_i}_{v^Tv}}.
\end{equation*} 
        \item [ii.] \begin{equation*}
\text{Tr}(D^TB_k)=\text{Tr}\big(\sum_{i=1}^k D^Tp_ix_i^T\big)
=\sum_{i=1}^k\text{Tr}\big( {\underbrace{D^Tp_i}_{a}}{\underbrace{x_i^T}_{b^T}}\big)
=\sum_{i=1}^k{\underbrace{x_i^TD^Tp_i}_{a^Tb}}
\end{equation*}
also
\begin{equation*}
\text{Tr}(B_k^TD)=\text{Tr}(D^TB_k)
=\sum_{i=1}^kp_i^TDx_i
\end{equation*}
which together lead to the result.
    \end{itemize}
\end{proof}
Let $F(D)=\frac{1}{2}\text{Tr}(D^TDA_k)-\text{Tr}(D^TB_k)$ in \eqref{e:algeq1}. 
Using the properties $\textit{P3-P5}$ of Proposition \ref{pro1} and the fact that $A_k=A_k^T$, we have 
\begin{eqnarray*}
\nabla_DF(D)&=&\nabla_D\big(\frac{1}{2}\text{Tr}(D^TDA_k)\big)-\nabla_D\big(\text{Tr}(D^TB_k)\big)\nonumber\\
&=&\frac{1}{2}\Big(\nabla_D\big(\text{Tr}(D^TDA_kI_{n})\big)\Big)-\nabla_D\big(\text{Tr}(B_k^TD)\big)\\
&=&\frac{1}{2}\Big(\nabla_D\big(\text{Tr}(DA_kI_{n}D^T)\big)\Big)-\nabla_D\big(\text{Tr}(DB_k^T)\big)\\
&=& \frac{1}{2}\Big(\nabla_D\big(\text{Tr}(DA_kD^TI_{m})\big)\Big)-B_k\\
&=&\frac{1}{2}\big(I_mDA_k+I_m^TDA_k^T\big)-B_k\\
&=&DA_k-B_k
\end{eqnarray*}
Now, $\nabla_DF(D)=0$ leads to \eqref{e:algeq2} column-wise.
To make it simple and clear, let assume that the square matrix $A_k$ is invertible. So, 
\begin{equation}\label{e:Ain}
 \nabla_DF(D)=0\Rightarrow D=B_kA_k^{-1}   
\end{equation}
We can rewrite \eqref{e:Ain} as
\begin{equation}\label{e:Ain1}
D=B_kA_k^{-1}=B_kA_k^{-1}-D+D=(B_k-DA_k)A_k^{-1}+D  
\end{equation}
Now, considering \eqref{e:Ain1} column-wise for nonzero diagonal entries of $A_k^{-1}$ leads to \eqref{e:algeq2}. 
\vspace{3mm}
\\

\subsubsection*{Acknowledgment:} We acknowledge funding by the German Aerospace Center
(DLR) as part of project AIMS: Artificial Intelligence Meets Space (grant number 50WK2270F).
%

\end{document}